\documentclass{article}
\usepackage{amsmath,amssymb,amsfonts,bm}
\usepackage{multirow}
\usepackage{mathrsfs}
\usepackage{booktabs}
\usepackage{authblk}
\usepackage[ruled,vlined]{algorithm2e}
\usepackage{ragged2e} % \justifying¿ÉÓÃ
\usepackage{booktabs}
\usepackage{amsthm}
\usepackage{latexsym}
\usepackage{graphicx}
\usepackage{rotating}
\usepackage{subfig}
\usepackage{color,varioref}
\usepackage{longtable}
\allowdisplaybreaks[4]
\usepackage[colorlinks=true,linkcolor=blue]{hyperref}

\newtheorem{theorem}{Theorem}[section]
\newtheorem{assumption}{Assumption}[section]
\newtheorem{definition}{Definition}[section]
\newtheorem{lemma}{Lemma}[section]
\newtheorem{remark}{Remark}[section]
\newtheorem{corollary}{Corollary}[section]
\newtheorem{scheme}{Scheme}[section]
\usepackage{geometry}
\usepackage{enumitem}

\usepackage{float}

\def\L{{\mathcal L}}
\def\R{{\mathbb  R}}
\def\bw{\mathbf{w}}
\def\bz{\mathbf{z}}
\def\by{{\mathbf y}}
\def\bx{{\mathbf x}}
\def\bg{{\boldsymbol  g}}
\def\blam{{\boldsymbol  \lambda}}
\def\tauk{\mathcal{N}^{\tau_{k+1}}}
\def\tk{\tau_{k+1}}
\def\oby{\overline{\mathbf y}}
\def\obx{\overline{\mathbf x}}

\def\ADMM{{\tt FedADMM}}
\def\FA{{\tt FedAvg}}
\def\FP{{\tt FedProx}}
\def\FAT{{\tt FedAlt}}
\def\FS{{\tt FedSim}}
\def\MA{{\tt MiFedADMM}}
\def\1A{{\tt 1-iFedADMM}}
\def\2A{{\tt 2-iFedADMM}}
\def\3A{{\tt 3-iFedADMM}}

\usepackage{titlesec}
\begin{document}
	
\title{A Multi-Step Inertial ADMM Algorithm for Federated Learning\footnote{This research is supported by the National Natural Science Foundation of China grants (12131004, 12401415, 12471282), the 111 Project (D23017), the Natural Science Foundation of Hunan Province (2025JJ60009), and the Postgraduate Scientific Research Innovation Project of Hunan Province (CX20250939, LXBZZ2024114).}}

\author[1]{Hao Li}
\author[1]{Qingsong Wang}
\author[2]{Xihong Yan}
\author[1]{Chengcheng Yan}
\author[3]{Deren Han\thanks{Corresponding author. Email: {\tt handr@buaa.edu.cn.}}}

\affil[1]{School of Mathematics and Computational Science, Xiangtan University, Xiangtan 411105, China}	
\affil[2]{Shanxi Key Laboratory for Intelligent Optimization Computing and Block-chain Technology, School of Mathematics and Statistics, Taiyuan Normal University, Jinzhong 030619, China}
\affil[3]{School of Mathematical Sciences, Beihang University, Beijing 100191, China}

\date{}	
\maketitle

\begin{abstract}
Federated learning has recently emerged as an important paradigm for large-scale distributed model training, which can be naturally formulated as a class of distributed learning problems. The alternating direction method of multipliers (ADMM) is a widely used approach for solving such problems. However, its practical performance is often hindered by significant communication overhead and system heterogeneity.
In this paper, we propose a multi-step inertial ADMM algorithm, called MiFedADMM, for federated learning. The proposed method leverages multiple historical iterates on the server side to construct a higher-order extrapolation after each aggregation to accelerate the convergence behavior. Under mild assumptions, we establish convergence results for the single-step variant of the proposed method and derive an $\mathcal{O}(1/K)$ convergence rate.
Numerical experiments on both synthetic and real-world datasets demonstrate the effectiveness and improved performance of the proposed method compared with existing approaches.
\end{abstract}
	
	
\maketitle
\section{Introduction}
Federated learning (FL) has become an important paradigm for distributed model training, where multiple clients collaboratively learn a global model under the coordination of a central server while keeping their local data locally \cite{konevcny2015federated, konevcny2016federated}. Such a paradigm can be naturally formulated as a class of distributed optimization problems. Owing to its privacy-preserving nature, FL has attracted considerable attention and has been widely applied in various areas such as mobile edge computing \cite{mao2017survey, yang2020federated, ye2022decentralized, zhou2021communication}, the Internet of Things \cite{9475501}, digital health \cite{rieke2020future}, and vehicular networks \cite{elbir2020federated,pokhrel2020federated, posner2021federated, samarakoon2019distributed}.

The typical FL training process, illustrated in Figure \ref{fig:FL}, consists of iterative communication between a central server and multiple clients. At every communication round (CR), each client first performs several local updates using its own dataset and uploads the updated parameters to the server. The server then aggregates these local parameters to produce a new global model and broadcasts it back to all clients.
\begin{figure}[htbp]
	\centering
	\includegraphics[width=\linewidth]{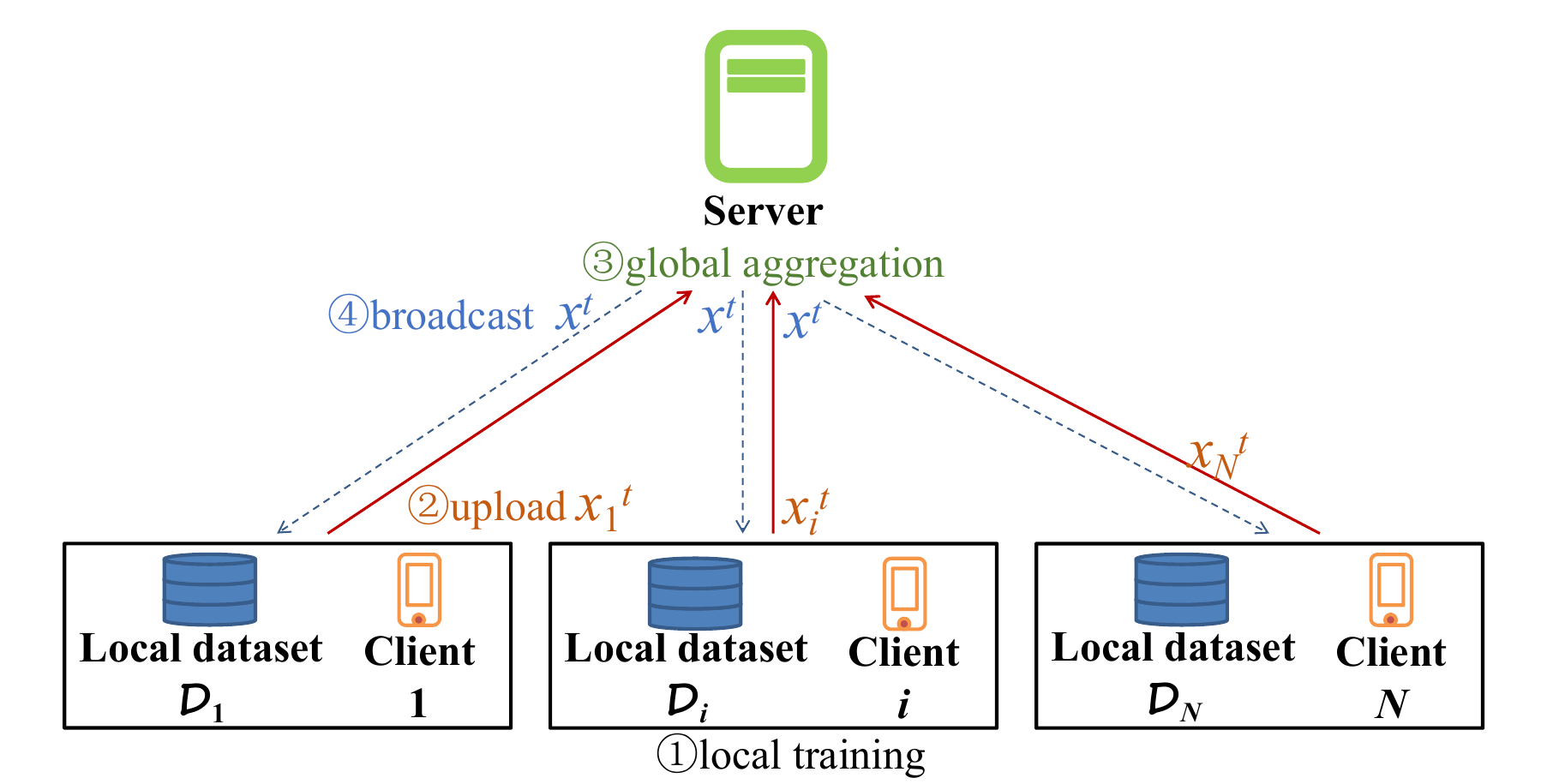}
	\caption{Process of FL.}
	\label{fig:FL}
\end{figure}
The above learning task can be formulated as the following problem. Consider a system with $N$ clients, each possessing a local dataset $\mathcal{D}_i$. The local loss function of client $i$ is defined as:
\begin{eqnarray*}
\begin{array}{lll}
f_i(\bx) := \frac{1}{|\mathcal{D}_i|}\sum\limits_{\xi\in\mathcal{D}_i} \ell_i(\bx; \xi),
\end{array}
\end{eqnarray*}
where $\xi$ denotes a data sample from $\mathcal{D}_i$, $\ell_i(\cdot; \xi_i): \mathbb{R}^n \to \mathbb{R}$ is a continuous and lower-bounded loss function, $|\mathcal{D}_i|$ denotes the cardinality of $\mathcal{D}_i$, and $\bx \in \mathbb{R}^n$ is the model parameter. The global objective is given by
\begin{eqnarray} \label{FL-opt}
\begin{array}{lll}
\underset\bx\min f(\bx):=\sum\limits_{i=1}^{N}  p_i f_i(\bx),
\end{array}
\end{eqnarray}
where $p_i > 0$ are weights satisfying $\sum_{i=1}^N p_i = 1$. 
While this formulation preserves data privacy by avoiding the transmission of raw data, multiple CRs may incur significant communication overhead, which becomes a major bottleneck for large-scale FL.

A foundational algorithm for FL is \FA\ \cite{mcmahan2017communication}, which aggregates local stochastic gradient descent (SGD) updates and has demonstrated strong empirical performance. However, under data heterogeneity (i.e., non-independent and identically distributed (non-i.i.d.) data across diverse clients), \FA\ suffers from client drift, where local models diverge toward their respective optima, thereby slowing down or even preventing convergence to a global optimum \cite{charles2021convergence,  karimireddy2020scaffold, li2019convergence, malinovskiy2020local,  mitra2021linear, pathak2020fedsplit}.
To alleviate this issue, several methods have been proposed. For instance, \FP\ \cite{li2020federatedprox} introduces a proximal term to penalize local deviations from the global model, though its performance is sensitive to the choice of the regularization parameter \cite{wang2022fedadmm}. {\tt SCAFFOLD} \cite{karimireddy2020scaffold} employs control variates to reduce the variance of local updates and correct the client drift. Nevertheless, it incurs additional communication overhead by requiring the transmission of auxiliary variables, which can be prohibitive in communication-limited environments. These limitations suggest that existing approaches mainly address client drift by modifying local updates, without fully exploiting the consensus-based structure of distributed learning.

Alternatively, the alternating direction method of multipliers (ADMM) \cite{boyd2011distributed,glowinski1975approximation, han2022survey} has recently attracted increasing attention as an effective framework for distributed learning in FL, owing to its theoretical guarantees and robustness to data heterogeneity. To apply ADMM, we reformulate problem \eqref{FL-opt} into the following consensus form by introducing auxiliary variables $\by_i=\bx, i\in[N]:=\{1,2,\cdots ,N\}$:
\begin{eqnarray} \label{FL-opt-ver1}
	\begin{array}{lll}
		\min\limits_{\bx, \by_i}~~\sum\limits_{i=1}^N  p_if_i(\by_i),\\
		{\rm s.t.}~~~\by_i=\bx,~~i\in[N].
	\end{array}
\end{eqnarray}
This formulation enables ADMM to decompose the problem into tractable local subproblems via dual variables and an augmented Lagrangian. In particular, the dual variables serve to correct the bias in local updates induced by non-i.i.d. data, while the quadratic penalty term enforces consensus among clients by balancing local objectives and global consistency. As a result, ADMM-based FL algorithms,   including {\tt FedPD} \cite{zhang2021fedpd}, \ADMM\ \cite{zhou2023federated}, and {\tt FedADMM-InSa} \cite{song2025fedadmm}, have demonstrated improved robustness to client drift compared with \FA.

However, to accommodate system heterogeneity (e.g., varying computational capabilities across clients), most existing ADMM-based FL methods rely on a fully synchronous round-based update scheme, which enforces strict coordination across clients at every CR. While such strategies ensure stability, they significantly increase the number of required CRs, resulting in substantial communication overhead and slow practical convergence, especially under limited communication budgets. This limitation becomes more pronounced in large-scale heterogeneous FL systems, where synchronization delays can severely degrade overall efficiency. 

To overcome the above limitations, we introduce a multi-step inertial technique into the global updates of {\tt FedADMM}, inspired by classical acceleration schemes \cite{anderson1965iterative, ortega1970iterative,nesterov1983method} and recent developments in inertial algorithms \cite{dong2019mikm, liang2016multi, zhang2020multi}. 
Specifically, at each CR, after aggregating local updates, the server performs a multi-step extrapolation on the aggregated global model. The extrapolation direction is constructed using multiple past global iterates, enabling the method to exploit historical information for accelerating convergence.
Importantly, the mechanism is implemented entirely on the server side, while keeping client-side update rules unchanged, incurring no additional communication cost per round, and posing no extra privacy risks, which makes it well-suited for practical heterogeneous FL scenarios.

It is worth distinguishing our approach from momentum-based methods (e.g., {\tt FedNAG} \cite{9891808}). Momentum methods operate on gradients by accumulating past gradients and performing a single-step extrapolation. 
In contrast, our multi-step inertial strategy constructs extrapolation directions by leveraging differences between successive past iterates, leading to a multi-step extrapolation mechanism that more effectively exploits historical information.
Moreover, the proposed approach does not require additional gradient tracking or accumulation, making it fundamentally different from momentum-based schemes.

More precisely, our main contributions are summarized as follows.

\begin{enumerate}[leftmargin=17pt, label=\alph*)]	
\item \textbf{Algorithm}. We propose \MA, a novel ADMM-based FL algorithm that incorporates a multi-step inertial acceleration mechanism into the server-side updates. Specifically, after each global aggregation, the server performs a multi-step linear extrapolation based on differences of multiple past aggregated iterates. This fully server-side design introduces no additional communication overhead and requires no modification to client-side computations.

\item \textbf{Convergence Analysis}. Under standard assumptions, including gradient Lipschitz continuity and coercivity, we establish global convergence of the proposed algorithm to a stationary point. Moreover, we derive a sublinear convergence rate of $\mathcal{O}(1/K)$ in terms of the stationarity measure $\min_{s=1,\dots,K}\|\nabla f(\bx^{\tau_{s}})\|^2$.

\item \textbf{Experimental Validation}. Extensive experiments on both synthetic datasets (linear regression) and real-world datasets (logistic regression) demonstrate that \MA\ significantly improves communication efficiency by reducing the number of required CRs, outperforming several state-of-the-art methods, including \ADMM, \FA, \FP, \FAT, and \FS\ \cite{pillutla2022federated}.
\end{enumerate}

The remainder of the paper is organized as follows. Section \ref{sec2} provides necessary preliminaries. Section \ref{sec3} presents the proposed algorithm and its key properties. In Section \ref{sec:convergence}, we present the convergence results, including global convergence and the $\mathcal{O}(1/K)$ sublinear convergence rate of our algorithm. Numerical experiments and conclusions are presented in Sections \ref{sec:num} and \ref{sec:conclusion}, respectively.

\section{Preliminaries}\label{sec2}
In this section, we introduce the notations frequently used in this paper, the optimality conditions for problems \eqref{FL-opt} and \eqref{FL-opt-ver1}, and the assumptions for the convergence proof of the proposed algorithm.
\subsection{Notation}
Scalars, vectors, and matrices are denoted by plain, bold, and capital letters, respectively (e.g., $k$, $\beta$; $\bx$, $\blam$; $Y$, $\Lambda$). Let $\lceil t \rceil$ denote the smallest integer greater than or equal to $t$ (e.g., $\lceil 1.1 \rceil = 2$, $\lceil 2 \rceil = 2$). Define $[N]:=\{1, 2,\cdots, N\}$ and let $\mathbb{N}_+ := \{1, 2, \cdots\}$ be the sequence of positive integers. The $n$-dimensional Euclidean space is denoted by $\R^n$, equipped with the inner product $\langle \cdot, \cdot \rangle$ defined as $\langle \bx, \blam \rangle := \sum_i^n x_i \lambda_i$. The induced 2-norm is denoted by $\|\cdot\|$, satisfying $\|\bx\|^2 = \langle \bx, \bx \rangle$. The gradient of a function $f$ with respect to $\bx$ is written as $\nabla f(\bx)$. For clarity, we aggregate groups of vectors $\by_i,\blam_i\in\R^n$ into the following matrix forms:
$$Y:=(\by_1,\by_2,\cdots ,\by_N),~~ \Lambda :=(\blam_1,\blam_2,\cdots ,\blam_N).$$
The same notation applies to their sequences, e.g., ($Y^k$, $\Lambda^k$) at iteration $k$, optima ($Y^*$, $\Lambda^*$), or accumulation points ($Y^\infty$, $\Lambda^\infty$), as referenced in Corollary \ref{L-global-convergence}. For simplicity, we also define the following objective function:
\begin{eqnarray*}
	\arraycolsep=1.0pt\def\arraystretch{1.5}
	\begin{array}{lll}
		F(Y):=\sum\limits_{i=1}^{N} p_i f_i(\by_i).
	\end{array}
\end{eqnarray*}
Note that $F(\bx,\bx,\cdots ,\bx)=f(\bx)$.
\subsection{Stationary point}
We define the stationarity conditions for problems \eqref{FL-opt-ver1} and \eqref{FL-opt}.
\begin{definition}\label{def-sta}
	A point $(\bx^*, Y^*,\Lambda^*)$ is called a stationary point of   problem  (\ref{FL-opt-ver1}) if it satisfies:
	\begin{eqnarray}\label{opt-con-FL-opt-ver1}
		\arraycolsep=1.0pt\def\arraystretch{1.5}
		\left\{\begin{array}{rcll}
			p_i\nabla  f_i(\by_i^*)+\blam_i^* &=&   0, ~~&i\in[N],  \\
			\by_i^*-\bx^* &=&0,&i\in[N],\\
			\sum\limits_{i=1}^{N} \blam_i^* &=& 0.
		\end{array} \right.
	\end{eqnarray}
\end{definition}
It can be verified that any local optimum of \eqref{FL-opt-ver1} satisfies \eqref{opt-con-FL-opt-ver1}. Furthermore, if each $f_i$ is convex for $i \in [N]$, then a point is a global optimum if and only if it meets condition \eqref{opt-con-FL-opt-ver1}.

Additionally, any stationary point $(\bx^*, Y^*,\Lambda^*)$ of \eqref{FL-opt-ver1} implies:
\begin{eqnarray}\label{grad-x-*=0}
	\arraycolsep=1.0pt\def\arraystretch{1.5}
	\begin{array}{rcll}
		\nabla   f(\bx^*) = \sum\limits_{i=1}^{N}   p_i \nabla   f_i(\bx^*) =  - \sum\limits_{i=1}^{N}  \blam_i^*=0.
	\end{array}
\end{eqnarray}
which shows that $\bx^*$ is also a stationary point of the original problem \eqref{FL-opt}.
\subsection{Assumptions}
We impose the following standard assumptions on the local and global loss functions, which ensure the well-definedness of problem \eqref{FL-opt} and underpin the convergence analysis of the proposed algorithm.
\begin{assumption}\label{ass-fi}
	Each local function $ f_i $, for $ i \in [N] $, is gradient Lipschitz continuous with constant $ L_i > 0 $; that is, for all $ \bx, \mathbf{z} \in \R^n $,
	\begin{align}\label{Lip-r}
		\| \nabla f_i(\bx) - \nabla f_i(\mathbf{z}) \| \leq L_i \| \bx - \mathbf{z} \|.
	\end{align}
\end{assumption}

\begin{assumption}\label{ass-f}
	The global function $ f $ is coercive, i.e.,
	\[
	\lim_{\| \bx \| \to \infty} f(\bx) = +\infty.
	\]
\end{assumption}

\section{Multi-step inertial ADMM for FL}\label{sec3}
This section details the proposed algorithm, including an efficient strategy for solving the resulting subproblems that do not admit closed-form solutions.
\subsection{Algorithmic design}
To solve problem \eqref{FL-opt-ver1} using the ADMM framework, we first define the augmented Lagrangian function:
\begin{eqnarray} \label{Def-L}
	\arraycolsep=0.5pt\def\arraystretch{1.5}
\begin{array}{lll}
\L(\bx,Y,\Lambda)&:=& \sum\limits_{i=1}^N  \L_i(\bx,\by_i ,\blam_i)\\
\L_i(\bx,\by_i ,\blam_i)&:=& p_i f_i(\by_i){+}    \langle \by_i{-}\bx, \blam_i\rangle {+} \frac{\beta_i}{2}\|\by_i{-}\bx\|^2,
\end{array}
\end{eqnarray}
where $\Lambda$ is the Lagrange multiplier, and $\beta _i>0,i\in[N]$. Define $\beta:=\sum_{i=1}^N\beta_i$. The standard ADMM iterations proceed as follows: given an initial  point  $(\bx^0, Y^0, \Lambda^0)$, for each $k \geq 0$,
\begin{eqnarray*}\label{framework-ADMM-equ}
	\arraycolsep=0.5pt\def\arraystretch{1.5}
	\left\{\begin{array}{llll}
		\bx^{k+1} &=&\arg\underset\bx\min\L(\bx,Y^k, \Lambda^k)
		=\frac{1}{\beta}\sum\limits_{i=1}^{N}  ( \beta _i\by^k_i+\blam^k_i), \\
		\by_i^{k+1} &=&\arg\underset{{\by_i}}\min\L_i(\bx^{k+1},\by_i, \blam_i^k),& i\in[N], \\
		\blam^{k+1}_i &=&\blam_i^k + \beta _i(\by_i^{k+1}-\bx^{k+1}),& i\in[N].
	\end{array}\right.
\end{eqnarray*}
In the FL setting, $\bx^{k+1}$ corresponds to the global update performed by the server, while $(\by_i^{k+1}, \blam_i^{k+1})$ are local updates computed by client $i$.
Based on the above framework, \ADMM\ proposed in \cite{zhou2023federated} offers three key advantages:
(i) reduction of communication overhead through multiple local updates,
(ii) alleviation of computational burden via inexact $\by_i^{k+1}$-subproblem solutions, and
(iii) mitigation of the straggler effect through partial clients participation.

The stragglers' effect arises from system heterogeneity, where clients possess diverse computational and communication capabilities. Slower clients delay the collective training process, which forces the entire system to wait for the slowest.

Building upon \ADMM, we further propose \MA, an accelerated ADMM-based FL algorithm that incorporates a multi-step inertial extrapolation mechanism into the server-side global updates, thereby improving convergence efficiency without increasing communication overhead. Figure~\ref{fig:MI} illustrates this technique via a three-step example, and the full implementation is provided in Algorithm~\ref{algorithm-ICEADMM}.

\begin{figure}[ht]
	\centering
	\includegraphics[width=\linewidth]{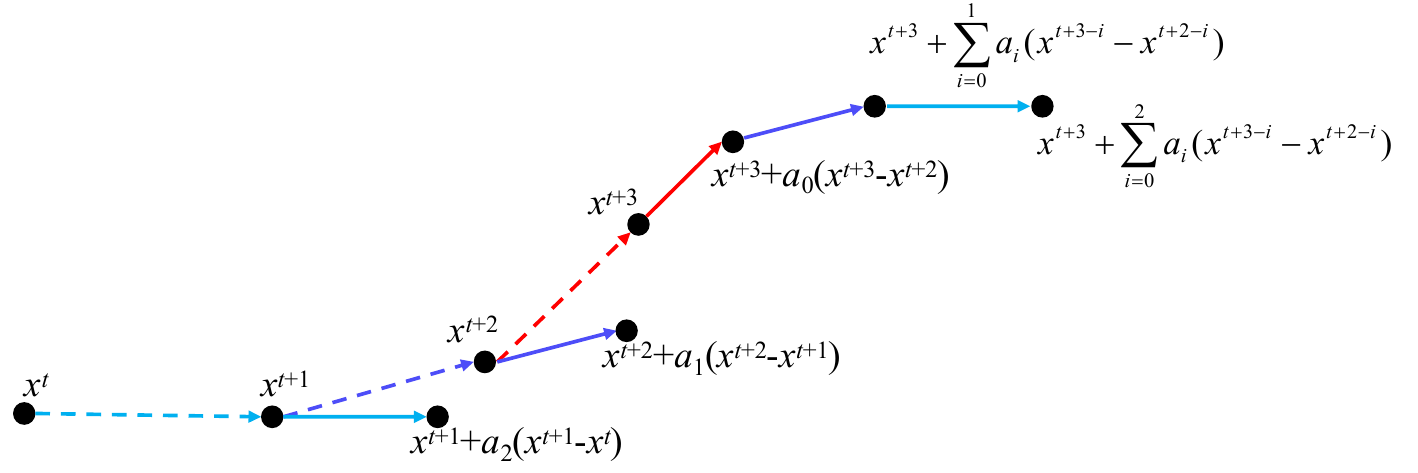}
	\caption{Multi-step inertial technique.}
	\label{fig:MI}
\end{figure}
%%%%%%%%%%%%%%%%%%%%%%%%%%%%%%%%%
\begin{algorithm*}[!th]
\caption{\MA}\label{algorithm-ICEADMM}
\SetAlgoLined
\KwIn{
$T,k_0,s\in\mathbb{N}_+$. $\hat\bx^0$. $\varepsilon_i^0$, $\beta_i$, $\kappa_i\in\mathbb{N}_+$, $\theta_i\in[1/2,1)$, $\forall i\in[N]$. $a_{i,\tk}\in (0, 1/2)$, $\text{for}~i = 0,1,\dots,s-1$, $\text{where}~\tk:=\lceil (k+1)/k_0\rceil~\text{denotes the CR index}$.
}
\KwOut{$\bx^{\tau_T}$.}
\textbf{Initialize}: $\mathcal{N}^0=[N]$. $\hat\bx^{-s}=\ldots=\hat\bx^{-1}=\hat\bx^0$, $\by_i^0=\hat\bx^0$, $\blam_i^0=-p_i\nabla f_i(\by_i^0)$,   $\mathbf{z}_i^0=\beta_i\by_i^0+\blam_i^0$, $\forall i\in[N]$.\\
\For{$k=0,1,2,\cdots,T-1$}{
\textbf{Global Updates:}\\
\If{$k\in\mathcal{K}:=\{0,k_0,2k_0,3k_0,\cdots\}$}{
Clients in $\mathcal{N}^{\tau_k}$ send their parameters $\{\mathbf{z}^k_i:i\in \mathcal{N}^{\tau_k}\}$  to the server,\
\vspace{-0.5em}
\begin{align}
\hat\bx^{\tk}&=\textstyle\frac{1}{\beta}\sum\limits_{i=1}^N  \mathbf{z}^k_i, \label{iceadmm-sub1}\\
\bx^{\tk}&=\textstyle\hat\bx^{\tk}+\sum\limits_{i=0}^{s-1} a_{i, \tk}\left(\hat\bx^{\tk-i}-\hat\bx^{\tk-i-1}\right),\label{accelerated}
\end{align}
\vspace{-1em}
			
The  server randomly selects clients in $[N]$ to form a subset $\tauk$ and broadcasts them parameter $\bx^{\tk}$.
		}
		
\textbf{Local Updates:}\\
		\For{every client $i\in \tauk$}{
			\vspace{-1em}
			\begin{align}
				\varepsilon_i^{k+1}&=\theta_{i}\varepsilon_i^k,\label{iceadmm-sub20}\\
				\by^{k+1}_i&=GD(p_i, L_i, \beta_i, \bx^{\tk}, \blam_i^k, \varepsilon_i^{k+1}, \kappa_i),\qquad\text{\# Algorithm \ref{algorithm2}}\label{iceadmm-sub2}\\
				\blam^{k+1}_i&=\blam_i^k+\beta_i(\by_i^{k+1}-\bx^{\tk}),\label{iceadmm-sub3}\\
				\mathbf{z}^{k+1}_i&=\beta _i\by_i^{k+1}+\blam_i^{k+1}.\label{iceadmm-sub4}
			\end{align}
			\vspace{-2em}
		}
		
		\For{every client $i\notin \tauk$}{
			\vspace{-1em}
			\begin{align}\label{iceadmm-sub5}
				(\varepsilon_i^{k+1}, \by^{k+1}_i,\blam^{k+1}_i,\mathbf{z}^{k+1}_i) =(\varepsilon_i^k, \by^k_i,\blam^k_i,\mathbf{z}^k_i).
			\end{align}
			\vspace{-2em}
		}

	}
\end{algorithm*}
\begin{remark}
For the case of $s=0$, where the multi-step inertial technique is effectively disabled (as there are no parameters to apply), \MA\ reduces to \ADMM.
\end{remark}
\begin{remark}
The multi-step inertial update \eqref{accelerated} introduces a modest computational overhead on the server side. Compared to \ADMM, it requires storing the past \(s\) aggregated iterates and performing \(O(s \cdot n)\) additional floating-point operations (FLOPs) per CR. This overhead is negligible for typical model sizes (e.g., \(n \leq 10^5\)) and is justified by the substantial gain in communication efficiency, i.e., reaching the target accuracy with significantly fewer CRs. This trade-off is particularly favorable in IoT edge networks, where communication is often the bottleneck.
\end{remark}

\subsection{Efficient inexact subproblem solving}
This part details the procedure for solving the local subproblem $\min_{\by_i} \L_i(\bx^{\tk},\by_i, \blam^k)$.  Notably, to ensure overall convergence while maintaining computational efficiency, an exact solution is not required. Instead, we adopt the following inexact surrogate:
\begin{eqnarray}\label{GD}
\arraycolsep=0.5pt\def\arraystretch{1.5}
\begin{array}{rcll}
\bw_i^{\ell+1}&=& \arg\min\limits_{\by_i}~
		\langle \by_i-\bx^{\tk}, \blam_i^k\rangle + \frac{\beta_i}{2}\|\by_i-\bx^{\tk}\|^2
	+p_i (f_i(\bw _i^{\ell}) + \langle \nabla f_i(\bw _i^{\ell}), \by_i- \bw _i^{\ell}  \rangle  + \frac{L_i}{2}\|\by_i-\bw _i^{\ell}\|^2) \\
		&=&\frac{1}{p_iL_i   +\beta_i}(p_iL_i \bw_i^{\ell} +\beta_i \bx^{\tk} - ( p_i\nabla f_i(\bw_i^{\ell}) +  \blam_i^k)),
\end{array}
\end{eqnarray}
where $\ell$ denotes the inner iteration index. This update corresponds to a gradient descent with step size $1/(p_i L_i + \beta_i)$ applied to the subproblem. The complete procedure is summarized in Algorithm~\ref{algorithm2}, which corresponds to \eqref{iceadmm-sub2} in Algorithm \ref{algorithm-ICEADMM}.

\begin{algorithm}[htbp]
	\caption{Gradient descent (GD) for the $\by_i^{k+1}$ subproblem \eqref{iceadmm-sub2}}\label{algorithm2}
	\SetAlgoLined
	\KwIn{$p_i, L_i, \beta_i, \bx^{\tk}, \blam_i^k, \varepsilon_i^{k+1}, \kappa_i$.}
	\KwOut{$\by_i^{k+1}$.}
	\textbf{Initialize}: $\bw_i^0=\bx^{\tk}$.\\
	\For{$\ell=0,1,2,\cdots,\kappa_i$}{
		
		compute $\bw_i^{\ell+1}$ via \eqref{GD}
		
		\If{$\left\|p_i\nabla f_i(\bw_i^{\ell+1})+\blam_i^k+\beta_i(\bw_i^{\ell+1}-\bx^{\tk})\right\|^2\leq\varepsilon_i^{k+1}$}{
			$\by_i^{k+1}=\bw_i^{\ell+1}$\;
			\textbf{break}\;
		}
	}
\end{algorithm}

In Algorithm~\ref{algorithm2}, $\kappa_i$ represents the maximum number of iterations for updating the sequence $\{\bw_i^{\ell}\}$. The other termination condition relies on evaluating the gradient of the augmented Lagrangian $\mathcal{L}_i$ with respect to $\bw_i$, evaluated at $\bw_i^{k+1}$, i.e.,
\begin{eqnarray*}
	\arraycolsep=1pt\def\arraystretch{1.5}
	\begin{array}{lll}
		\nabla_{\bw_i} \mathcal{L}_i ( \bx^{\tk}, \bw_i^{k+1}, \blam^k ) 
		= \! p_i\nabla f_i(\bw_i^{k+1}) + \blam_i^k + \beta_i (\bw_i^{k+1} - \bx^{\tk}),
	\end{array}
\end{eqnarray*}
which is checked against the tolerance $\varepsilon_i^{k+1}$. By definition of $\bw_i^*$,  we have $\nabla_{\bw_i} \mathcal{L}_i(\bx^{\tk}, \bw_i^*, \blam^k)= 0$.

The following theorem establishes that Algorithm~\ref{algorithm2} computes an approximate solution $\by_i^{k+1}$ satisfying a predefined accuracy within at most $\kappa_i$ iterations.
\begin{theorem}\label{finite-stop}
	Suppose Assumption \ref{ass-fi} holds, $ \nabla^2 f_i \succeq -\sigma_i I$ with $\sigma_i\geq0$, and $\beta_i\geq p_i \sigma_i+c  p_iL_i/2 $ with $c>1$. Then  client $i\in[N]$ can compute $\by_i^{k+1}$ using Algorithm \ref{algorithm2} in at most $\kappa_i$ iterations, where
	\begin{eqnarray*}
		\arraycolsep=1.0pt\def\arraystretch{1.5}
		\begin{array}{rcll}
			\kappa_i = \log_{c}\left \lceil \frac{2(p_i^2L_i^2+\beta_i^2)\|\bx^{\tk}-\bw_i^*\|^2}{\varepsilon_i^{k+1}}  \right \rceil -1.
		\end{array}
	\end{eqnarray*}
\end{theorem}
\begin{proof}
	The proof follows directly from the argument in \cite[Theorem 3.1]{zhou2023federated}.
\end{proof}
\begin{remark}
	$\nabla^2 f_i \succeq -\sigma_i I$  means $\nabla^2 f_i + \sigma_i I\succeq 0$, which is a positive semi-definite matrix. This standard assumption implies that the minimum eigenvalue of $\nabla^2 f_i$ is bounded below by $-\sigma_i$.  It holds for all convex functions (with $\sigma_i = 0$) and numerous non-convex functions.
\end{remark}
\section{Convergence analysis}\label{sec:convergence}
We establish the global convergence and convergence rate of \MA\ for the case of $ s=1 $. In this setting, the inertial parameter set reduces to a single element $ a_{0, \tk} $, which we denote simply as $ a_{\tk} $ for notational simplicity.
In addition to Assumptions \ref{ass-fi} and \ref{ass-f}, we introduce the following client selection scheme.
\begin{scheme}\label{ass-omega}
There exists a constant $s_0 \in \mathbb{N}_+$ such that each client is selected at least once within any consecutive $s_0$ CRs, i.e.,
for any $k \ge 0$ and any client $i \in [N]$, there exists an index 
$t \in \{k+1, \dots, k+s_0\}$ such that $i \in \mathcal{N}^t$.
\end{scheme}
\begin{remark}\label{rem:omega}
Scheme~\ref{ass-omega} can be approximately satisfied under standard random sampling strategies. 
In particular, suppose that at each round, a subset $\mathcal{N}^t$ is selected uniformly at random with a fixed participation rate $\rho = |\mathcal{N}^t|/N$, 
and the selections are independent across rounds. 
Then, for any client $i \in [N]$, the probability that it is selected at least once within $s_0$ consecutive rounds is given by
\[
1 - (1 - \rho)^{s_0},
\]
which approaches $1$ as $s_0$ increases. 
Therefore, Scheme~\ref{ass-omega} can be viewed as a mild assumption that holds with high probability in practical implementations.
\end{remark}
For notational simplicity, we denote
\begin{align*}
    \triangle\by_i^{k+1}       &:= \by_i^{k+1} - \by_i^k, &
    \triangle\oby_i^{k+1}      &:= \by^{k+1}_i - \bx^{\tk}, &
    \triangle\blam_i^{k+1}     &:= \blam_i^{k+1} - \blam_i^k, &
    \triangle\varepsilon_i^{k+1}&:= \varepsilon_i^{k+1} - \varepsilon_i^k, \\
    \triangle\bx^{\tk}         &:= \bx^{\tk} - \bx^{\tau_k}, &
    \triangle\hat\bx^{\tk}     &:= \hat\bx^{\tk} - \hat\bx^{\tau_k}, &
    \triangle\obx^{\tk}        &:= \bx^{\tk} - \hat\bx^{\tk},
\end{align*}
and let   $\bx^k \to \bx$ denote $\lim\limits_{k\to\infty} \bx^k = \bx$. Additionally, we write
\begin{eqnarray*}
	\begin{array}{rcl}
		\sum_i:=\sum\limits_{i=1}^N,\quad
		\bg_i^k:=p_i\nabla f_i(\by_i^k).
	\end{array}
\end{eqnarray*}
\subsection{Global convergence}
To establish the convergence of \MA, we first introduce the following pivotal lemma.
\begin{lemma}\label{basic-observations}
The following properties hold for the sequences generated by \MA:
	\begin{enumerate}[leftmargin=17pt, label=\alph*)]
		\item For $k\in\mathcal{K}$,
		\begin{eqnarray}\label{opt-con-xk1-1}
			\arraycolsep=1.0pt\def\arraystretch{1.5}
			\begin{array}{lcl}
				\sum_i  ( \beta_i(\by_i^k-\hat\bx^{\tk})+\blam_i^k)=0.
			\end{array}
		\end{eqnarray}
		
		\item For  any  $k\geq 0$ and $i\in[N]$,
		\begin{align}\label{opt-con-xk1-2}
			{\|\Phi_i^{k+1}\|^2\leq \varepsilon_i^{k+1},~~\text{where}~~ \Phi_i^{k+1}:=\bg_i^{k+1}  +  \blam_i^{k+1}.}
		\end{align}
		
		\item Under Assumption \ref{ass-fi}, for  any  $k\geq 0$ and $i\in[N]$,
		\begin{eqnarray}\label{opt-con-xk1-3}
			\arraycolsep=1.0pt\def\arraystretch{1.5}
			\begin{array}{lcl}
				\|\triangle \blam_i^{k+1} \|^2  \leq   \frac{6p_i^2L_i^2}{5} \|\triangle\by_i^{k+1} \|^2 -\frac{24}{1-\theta_{i}}\triangle\varepsilon_i^{k+1}.
			\end{array}
		\end{eqnarray}
		
		\item Under Scheme \ref{ass-omega}, for $i\in[N]$,
		\begin{align}\label{opt-con-xk1-4}
			\varepsilon_i^{k+1}\to 0.
		\end{align}
	\end{enumerate}
\end{lemma}
\begin{proof}
See Appendix \ref{lemma1} for detailed proof.
\end{proof}
From Lemma \ref{basic-observations}, we derive the following key decreasing property of the sequence $\{\widetilde\L^k\}$, which underlies our subsequent convergence analysis.
\begin{lemma}\label{lemma-decreasing-1}
Suppose Assumption~\ref{ass-fi} holds and each client $i \in [N]$ selects $\beta_i \geq 3p_iL_i$. 
Define the Lyapunov function
\begin{eqnarray}\label{def-c}
	\arraycolsep=0.5pt\def\arraystretch{1.5}
	\begin{array}{lll}
		\widetilde\L^k &:=& \L^k + \sum_i \frac{29 \varepsilon_i^k}{(1-\theta_i)\beta_i}
		+ \frac{\beta}{2}(a_{\tau_k} - a_{\tau_k}^2)\|\hat\bx^{\tau_k} - \hat\bx^{\tau_k-1}\|^2,\\
		\L^k &:=& \L(\bx^{\tau_k}, Y^k, \Lambda^k).
	\end{array}
\end{eqnarray}
Then the sequence $\{\widetilde{\mathcal L}^k\}$ satisfies
\begin{eqnarray*}
	\arraycolsep=1.0pt\def\arraystretch{1.5}
	\begin{array}{rcl}
		\widetilde\L^k - \widetilde\L^{k+1}
		&\geq&
		\sum_i \frac{\beta_i}{10} \|\by_i^{k+1} - \by_i^k\|^2
		+ \frac{\beta}{2}(1 - a_{\tk} - a_{\tau_k}) \|\triangle \hat\bx^{\tk}\|^2.
	\end{array}
\end{eqnarray*}
\end{lemma}
\begin{proof}
A detailed proof is provided in Appendix \ref{lemma2}.
\end{proof}
Lemma~\ref{lemma-decreasing-1} not only implies the non-increasing nature of $\{\widetilde{\mathcal{L}}^k\}$ but also the convergence to zero of the successive differences $\triangle \bx^{\tk}$, $\triangle\oby_i^{k+1}$, $\triangle \by_i^{k+1}$, $\triangle\blam_i^{k+1}$, $\triangle\hat\bx^{\tk}$, and $\triangle\obx^{\tk}$. Formal statements of these convergence results follow.
\begin{lemma}\label{L-bounded-decreasing-1}
Suppose Assumption \ref{ass-fi}  holds, each client  $i\in[N]$ selects $\beta_i\geq 3p_iL_i$, and the server follows Scheme \ref{ass-omega}. Then the following hold:
	\begin{enumerate}[leftmargin=17pt, label=\alph*)]
		\item The sequence $\{ \widetilde\L^k\}$ is non-increasing.
		
		\item For $k\geq1$, $\widetilde\L^k\geq f(\bx^{\tau_k})\geq f(\bx^*) >-\infty$.
		
		\item The following quantities satisfy:
		\begin{align}\label{limit-5-term-0-inexact}
			(\varepsilon_i^{k+1}, \triangle \bx^{\tk}, \triangle\oby_i^{k+1}, \triangle \by_i^{k+1}, \triangle \blam_i^{k+1}, 
			\triangle \hat\bx^{\tk}, \triangle\obx^{\tk}) \to 0.
		\end{align}
	\end{enumerate}
\end{lemma}
\begin{proof}
Refer to Appendix \ref{lemma3} for a detailed proof.
\end{proof}
Lemma \ref{L-bounded-decreasing-1} ensures that the sequence $\{\widetilde{\mathcal{L}}^k\}$ is non-increasing, bounded below, and that its successive differences vanish. These properties, in particular, guarantee the convergence of the sequences $\{f(\bx^{\tau_k})\}$, $\{F(Y^k)\}$, and $\{\mathcal{L}^k\}$, as formalized in the next theorem.
\begin{theorem}\label{global-obj-convergence-inexact}
Suppose Assumptions \ref{ass-fi} and \ref{ass-f} hold, each client $i\in[N]$ selects $\beta_i\geq 3p_iL_i$, and the server follows Scheme \ref{ass-omega}.  Then:
	\begin{enumerate}[leftmargin=17pt, label=\alph*)]
		\item The sequence $\{(\bx^{\tau_k},Y^k,\Lambda^k)\}$  is bounded.
		
		\item Three  sequences $\{\mathcal{L}^k\}$, $\{ F(Y^k)  \}$, and $\{f (\bx^{\tau_k})\}$ converge to the same value:
		\begin{eqnarray*}%\label{L-local-convergence-limit-inexact}
			\arraycolsep=1.0pt\def\arraystretch{1.0}
			\begin{array}{lll}
				\lim\limits_{k\to \infty} \mathcal{L}^k =  \lim\limits_{k\to \infty} F(Y^k)  = \lim\limits_{k\to \infty} f(\bx^{\tau_k}).
			\end{array}
		\end{eqnarray*}
		
		\item $\nabla F(Y^k)$ and $\nabla f(\bx^{\tau_k})$ eventually vanish:
		\begin{eqnarray*}%\label{L-local-convergence-limit-grad-inexact}
			\arraycolsep=1.0pt\def\arraystretch{1.0}
			\begin{array}{lll}
				\lim\limits_{k \to \infty}\nabla F(Y^k)  =\lim\limits_{k \to\infty} \nabla f(\bx^{\tau_k}) =0.
			\end{array}
		\end{eqnarray*}
	\end{enumerate}
\end{theorem}
\begin{proof}
The detailed proof can be found in Appendix \ref{theorem1}.
\end{proof}
While Theorem~\ref{global-obj-convergence-inexact} established the convergence of objective values, we now prove the convergence of the iterates ${(\bx^{\tau_k}, Y^k, \Lambda^k)}$ themselves, which requires slightly stronger conditions.

\begin{theorem}\label{global-convergence-inexact}
Suppose Assumptions \ref{ass-fi} and \ref{ass-f} hold, each client $i\in[N]$ selects $\beta_i\geq 3p_iL_i$, and the server follows Scheme \ref{ass-omega}.  Then:
	\begin{enumerate}[leftmargin=17pt, label=\alph*)]
		\item Any accumulating point $(\bx^{\infty}, Y^{\infty}, \Lambda^{\infty})$ of sequence $\{(\bx^{\tau_k}, Y^k,\Lambda^k)\}$ is a stationary point of  problem (\ref{FL-opt-ver1}), where $\bx^{\infty}$ is a stationary point of problem (\ref{FL-opt}).
		
		\item If further assuming that $\bx^{\infty}$ is isolated, then the whole sequence  converges to $(\bx^{\infty},Y^{\infty},\Lambda^{\infty})$.
	\end{enumerate}
\end{theorem}
\begin{proof}
	See Appendix \ref{theorem2} for the complete proof.
\end{proof}
\begin{remark}
The convergence guarantees in Theorems~\ref{global-obj-convergence-inexact} and~\ref{global-convergence-inexact} do not explicitly depend on the choice of $\mathcal{N}^{\tau_k}$ due to Scheme~\ref{ass-omega}. In particular, if $\mathcal{N}^{\tau_k}$ is generated randomly instead of deterministically satisfying Scheme~\ref{ass-omega}, the conclusions still hold with high probability.
\end{remark}

With the additional assumption of convexity, Theorems \ref{global-obj-convergence-inexact} and~\ref{global-convergence-inexact} imply the following corollary.

\begin{corollary}\label{L-global-convergence}
Suppose Assumptions \ref{ass-fi} and \ref{ass-f} hold, $f$ is convex, each client $i\in[N]$ selects $\beta_i\geq 3p_iL_i$, and the server follows Scheme \ref{ass-omega}.  Then:
	\begin{enumerate}[leftmargin=17pt, label=\alph*)]
		\item Three  sequences $\{\mathcal{L}^k\}$, $\{ F(Y^k)  \}$, and $\{f (\bx^{\tau_k})\}$ all converge to the optimal value of problem (\ref{FL-opt}):
		\begin{align}\label{L-global-convergence-limit}
			\lim_{k\to \infty} \L^k =  \lim_{k\to \infty} F(Y^k)  = \lim_{k\to \infty} f(\bx^{\tau_k})=  f(\bx^*).
		\end{align}
		
		\item Any accumulating point  $(\bx^{\infty},Y^{\infty},\Lambda^{\infty})$ of   sequence  $\{(\bx^{\tau_k},Y^k,\Lambda^k)\}$ is an optimal solution to problem (\ref{FL-opt-ver1}), where $\bx^{\infty}$ is an optimal solution to  problem  (\ref{FL-opt}).
		
		\item If $f$ is strongly convex, then whole sequence  converges the unique optimal solution $(\bx^*,Y^*,\Lambda^*)$ to   problem (\ref{FL-opt-ver1}), where $\bx^*$ is the unique optimal solution to   problem (\ref{FL-opt}).
	\end{enumerate}
\end{corollary}
\begin{proof}
	For a detailed proof, see Appendix \ref{corollary1}.
\end{proof}
\begin{remark}
The implication that strong convexity entails coercivity allows us to omit Assumption~\ref{ass-f} in the strongly convex case.
\end{remark}

\subsection{Sublinear rate of convergence}
We now establish the sublinear rate of convergence for our algorithm. That is, we prove that the algorithm  needs at most $\mathcal{O}(1/\epsilon)$ iterates to reach an $\epsilon$ solution.
\begin{theorem}\label{complexity-thorem-gradient-inexact}
Suppose Assumption \ref{ass-fi} holds, each client $i\in[N]$ selects $\beta_i\geq 3p_iL_i$, and the server follows Scheme \ref{ass-omega}.  Then for $k>2s_0k_0$,
$$\min_{s=1,2,\cdots,k}\|\nabla f(\bx^{\tau_{s +1}})\|^2  \leq    \frac{2c k_0}{k},$$
where $c>0$ is a constant (specified in the proof) independent of $k$.
\end{theorem}
\begin{proof}
We defer the detailed proof to Appendix \ref{theorem3}.
\end{proof}
\begin{remark}
The convergence rate is derived only under the gradient Lipschitz continuity assumption (Assumption~\ref{ass-fi}), without requiring convexity or other structural conditions.
\end{remark}

\section{Numerical experiments}\label{sec:num}
This section presents numerical experiments to evaluate \MA. 
We compare its performance with several baseline FL methods, including \ADMM, \FA, \FP, \FAT, and \FS, using the objective value $f(\bx^{\tau_k})$, CR, and computational time (seconds) as metrics. Additionally, we examine the influence of four key parameters: the number of local iterations per CR, $k_0$; the client participation rate, $\rho \in (0, 1]$; the total number of clients, $N$; and the dimension of the solution, $n$.

To investigate these parameters, we adopt the following experimental setup. For problem \eqref{FL-opt}, we set $p_i = 1/N$ for all $i \in [N]$ and maintain this value throughout Sections \ref{linear} and \ref{ex-logist}. All algorithms are initialized at $\bx^0 = \by_i^0 = \mathbf{0}$ for $i \in [N]$. We simulate partial clients participation by randomly and uniformly selecting a subset $\mathcal{N}^\tau \subseteq [N]$ of size $\rho N$ in each round $\tau$. This stochastic scheme satisfies Scheme~\ref{ass-omega} with high probability, as established in Remark~\ref{rem:omega}.

\begin{enumerate}[leftmargin=17pt, label=\alph*)]
	\item \ADMM: Initialize $\by_i^0 = \blam_i^0 = \mathbf{0}$. For $i \in [N]$, set $\varepsilon_i^0 = k_0^2$, $\theta_i = 0.95$, and $\beta_i = 0.5 L_i / N$. The algorithm terminates when
	\begin{eqnarray}\label{stopping}
		\arraycolsep=1.0pt\def\arraystretch{1.5}
		\begin{array}{lll}
			\| \nabla f(\bx^{\tau_k}) \|^2 < \min\left\{ \frac{1}{5} \| \nabla f(\mathbf{0}) \|^2, \frac{5 \delta n N}{d} \right\},
		\end{array}
	\end{eqnarray}
	where $d$ is specified in Sections \ref{linear} and \ref{ex-logist}, with $\delta$ set to $10^{-3}$ and $10^{-7}$ for the respective sections.
	
	\item \MA: Initialization and parameter settings are the same as \ADMM. We implement three variants of \MA:
	\begin{itemize}
		\item \1A: $s = 1$, $r_{\tk} = \frac{k+2}{2}$, $a_{0, \tk} = 0.4\times\frac{r_{\tk-1}-1}{r_{\tk}}$, and $a_{j, \tk} = 0$ for $j \geq 1$.
		
		\item \2A: $s = 2$, $r_{\tk} = \frac{k+2}{2}$, $a_{0, \tk} = a_{1, \tk} = 0.4\times\frac{r_{\tk-1}-1}{r_{\tk}}$, and $a_{j, \tk} = 0$ for $j \geq 2$.
		
		\item \3A: $s = 3$, $r_{\tk} = \frac{k+2}{2}$, $a_{0, \tk} = a_{1, \tk} = a_{2, \tk} = 0.4\times\frac{r_{\tk-1}-1}{r_{\tk}}$, and $a_{j, \tk} = 0$ for $j \geq 3$.
	\end{itemize}

	\begin{remark}
	The coefficients \(a_{i,\tk}\) are designed to satisfy two criteria. 
		First, to ensure convergence, we require \(a_{i,\tk} \in (0, 1/2)\) for 
		all \(i\) and \(\tk\), as implied by the convergence condition in Lemma \ref{lemma-decreasing-1}. Second, to achieve acceleration, we adopt a diminishing schedule \(a_{i,\tk} = \alpha \cdot (r_{\tk-1} - 1)/r_{\tk}\) with \(r_{\tk} = (k+2)/2\), which is inspired by Nesterov's momentum schedule \cite{nesterov1983method} and ensures that the coefficients decay as training progresses. The scaling factor \(\alpha = 0.4\) is chosen to stay safely below the theoretical upper bound \(1/2\), and empirical studies (see Section \ref{sec:num}) show that the algorithm is robust to \(\alpha\) in a range around \(0.4\). For the multi-step cases (\(s=2,3\)), we set all coefficients to the same value for a fair comparison, as our focus is to demonstrate the benefit of incorporating multiple history steps rather than optimizing individual coefficients.
	\end{remark}
	
	\item \FA\ and \FP: Each selected client $i \in \tauk$ approximates the local subproblem via gradient descent, initialized at $\bx^{\tau_k}$ if $k \in \mathcal{K}$ and at $\by_i^k$ otherwise.
	
	\item \FAT\ and \FS: The two methods employ the partial model personalization from \cite{hanzely2021personalized}:
	\begin{eqnarray*}
		\arraycolsep=1.0pt\def\arraystretch{1.5}
		\begin{array}{lll}
			h_i(\bx, \mathbf{v}) = (1 - \alpha) f_i(\bx) + \alpha f_i(\mathbf{v}) + \frac{\mu}{2} \| \bx - \mathbf{v} \|^2,
		\end{array}
	\end{eqnarray*}
	with $\alpha = 0.5$ and $\mu = 0.001$. Their local subproblems are initialized and solved identically to their counterparts in \FA\ and \FP.
\end{enumerate}

These methods conduct global averaging exclusively when $k \in \mathcal{K}$. Unlike \MA, \ADMM, and \FA, which aggregate across all clients, methods like \FP, \FAT,~and \FS~ average only over the participating clients, i.e.,
\begin{eqnarray*}
	\arraycolsep=1.0pt\def\arraystretch{1.5}
	\begin{array}{lll}
		\bx^{\tau_k} = \frac{1}{|\mathcal{N}^{\tau_k}|} \sum\limits_{i \in \mathcal{N}^{\tau_k}} \by_i^k.
	\end{array}
\end{eqnarray*}

For a fair comparison, we employ the following stopping criterion: first, \ADMM\ is run until it satisfies condition \eqref{stopping}, producing a solution $\bx^{\tau_k}$. All other algorithms terminate once their solution $\bx$ satisfies:
\begin{eqnarray*}
	\arraycolsep=1.0pt\def\arraystretch{1.5}
	\begin{array}{lll}
		f(\bx) - f(\bx^{\tau_k}) \leq 2(1 + |f(\bx^{\tau_k})|) \cdot 10^{-4},
	\end{array}
\end{eqnarray*}
This ensures consistent solution quality across all methods.

All experiments were conducted using MATLAB on a desktop equipped with an Intel Core i7-12700 processor (12th Gen, 12 cores, 2.10 GHz) and 16 GB RAM.
\subsection{Linear regression}\label{linear}

In this experiment, each client $i$ has the local objective function:
\begin{eqnarray*}
	\begin{array}{lll}
		f_i(\bx) = \frac{1}{2d_i} \sum\limits_{t=1}^{d_i} (\langle \mathbf{a}_i^t, \bx \rangle - b_i^t)^2,
	\end{array}
\end{eqnarray*}
where $\mathbf{a}_i^t \in \mathbb{R}^n$ and $b_i^t \in \mathbb{R}$ represent the $t$-th feature vector and label for client $i$. To simulate non-i.i.d. data,  we first draw dataset sizes $d_1, \dots, d_N$ from a uniform distribution over $[50, 150]$, and set $d = \sum_{i=1}^N d_i$. 
A global dataset is then created, comprising $\lceil d/3 \rceil$ samples from a standard normal distribution, $\lceil d/3 \rceil$ samples from Student's $t$-distribution with 5 degrees of freedom, and the remaining $d - 2\lceil d/3 \rceil$ samples from a uniform distribution on $[-5, 5]$. 
All samples are shuffled and partitioned into $N$ subsets with sizes $d_1, \dots, d_N$, assigned to each client accordingly. 

\subsubsection{Impact of \texorpdfstring{$k_0$}{k0}}
To evaluate the impact of $k_0$, we fix  $(n,N,\rho){=}(100, 100,0.5)$. First, we perform all algorithms under $k_0\in\{1,10,30,50,70,90\}$ and report the results in Figure  \ref{LinReg-k0}. The figure reveals that all methods eventually converge to the same objective value with increasing CR, so subsequent results omit objective values. For $k_0=1$, \MA~exhibits slow initial progress but accelerates significantly in later stages. For $k_0>1$, \MA~consistently achieves convergence in the fewest CR. Moreover, \3A outperforms \2A, which in turn surpasses \1A, highlighting the advantage of multi-step inertial updates.
\begin{figure}[htbp]
\centering
\includegraphics[width=.49\linewidth]{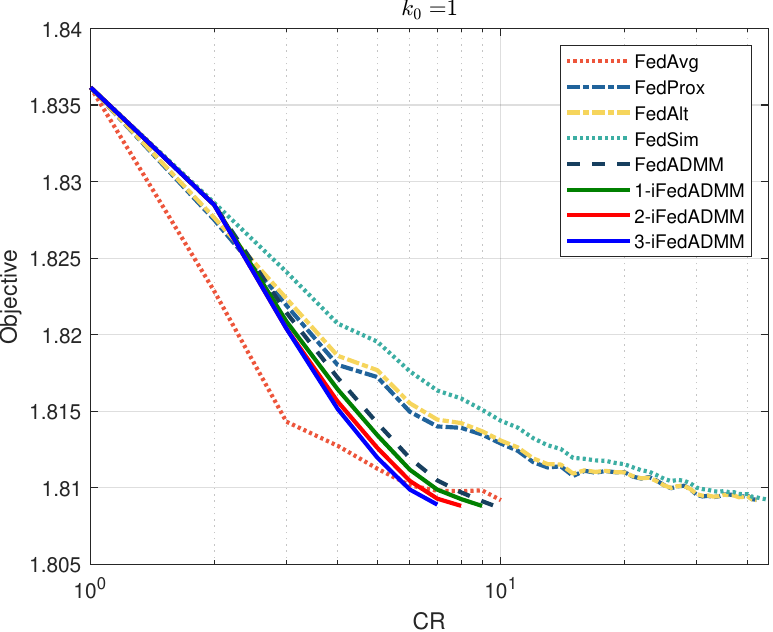}
\includegraphics[width=.49\linewidth]{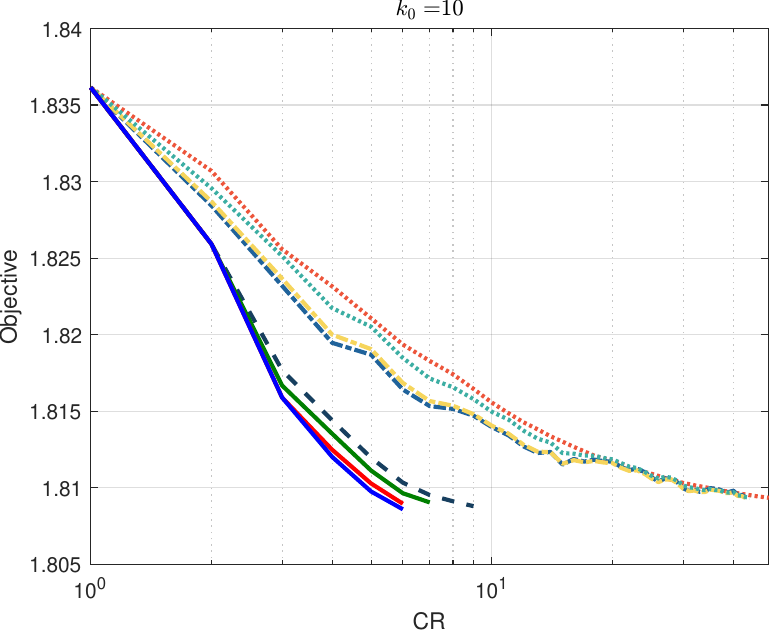}\\
\includegraphics[width=.49\linewidth]{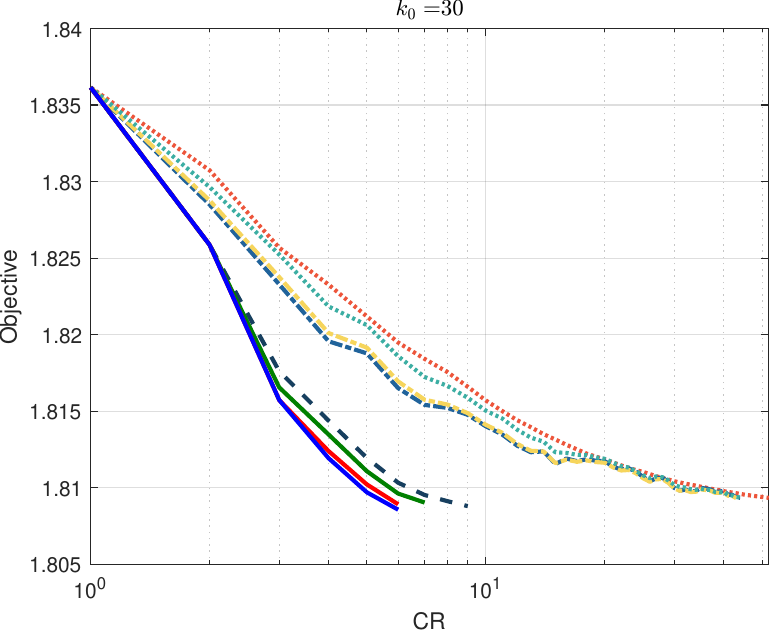}
\includegraphics[width=.49\linewidth]{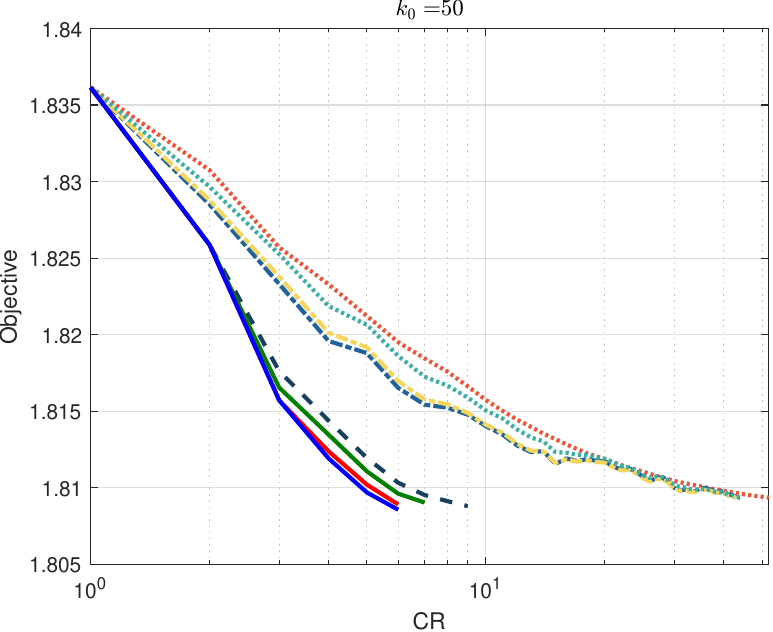}\\
\includegraphics[width=.49\linewidth]{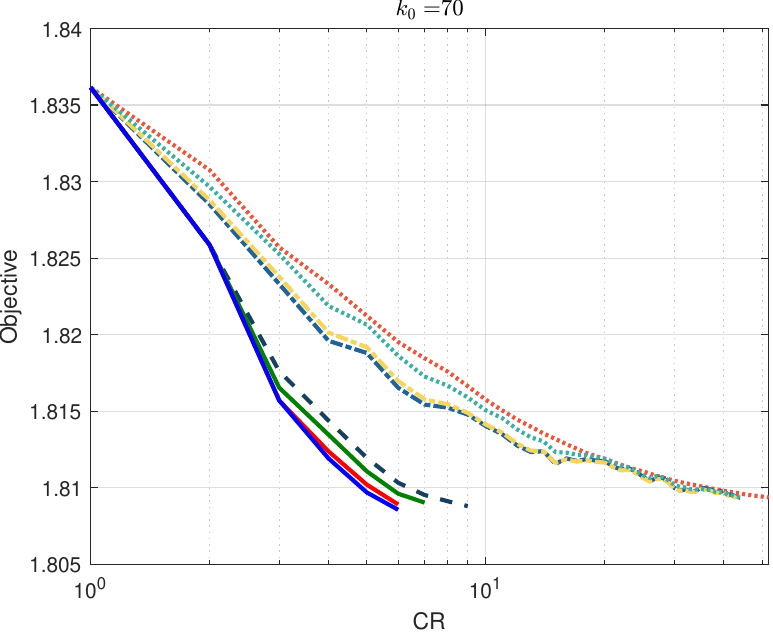}
\includegraphics[width=.49\linewidth]{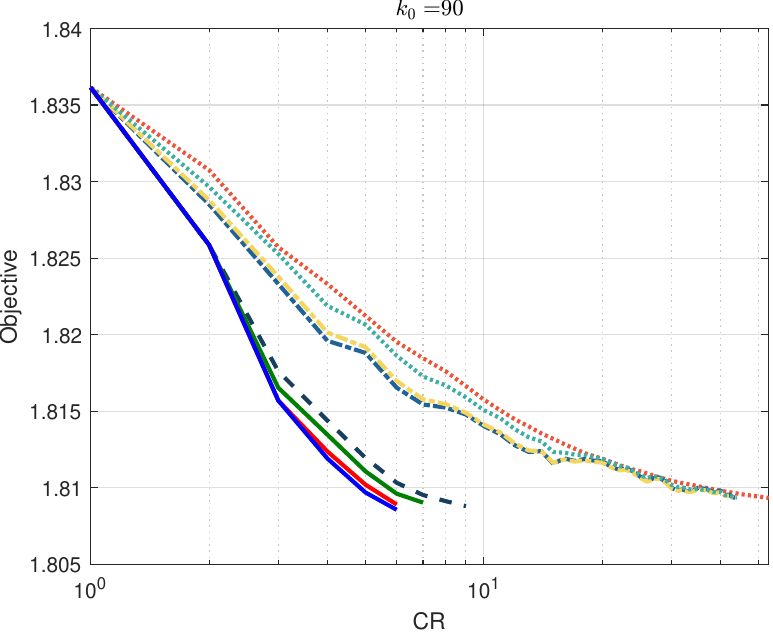}
\caption{Objective Value vs. CR.\label{LinReg-k0}}
\end{figure}

\begin{figure}[htbp]
\centering
\includegraphics[width=.49\linewidth]{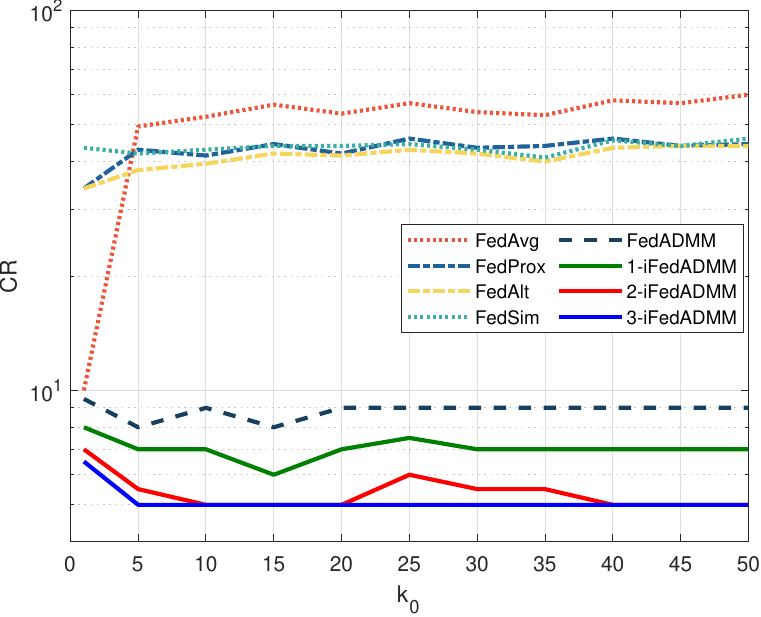}
\includegraphics[width=.49\linewidth]{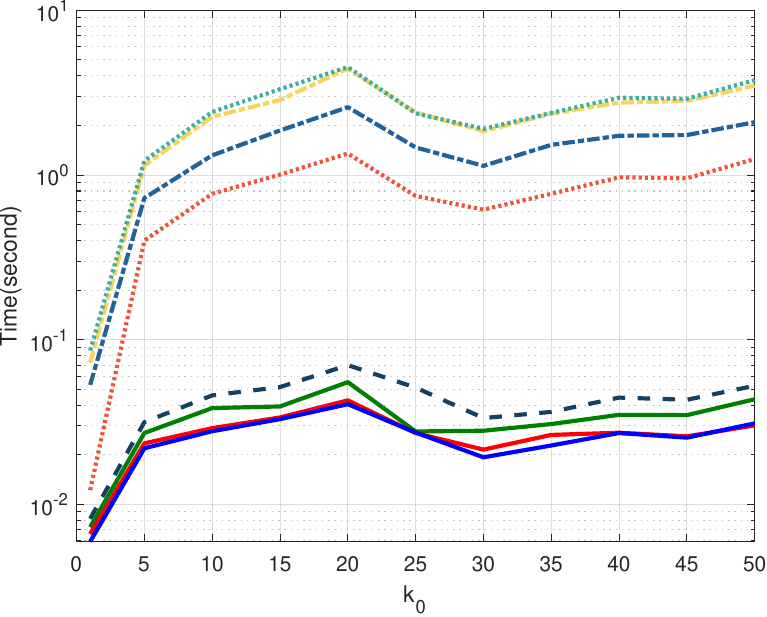}
\caption{Impact of $k_0$.\label{LinReg-k02}}
\end{figure}

In the second experiment, we investigate statistical reliability by repeating each run across 20 independent instances for each $k_0\in\{1,5,10,\cdots,50\}$, reporting the median results in Figure~\ref{LinReg-k02}. The results indicate that CR remains relatively stable for \MA~as $k_0$ increases. Although larger $k_0$ values lead to longer computation time, \MA~maintains the lowest CR and shortest runtime among all competitors. The consistent outperformance of \3A over \2A and \1A further validates the efficacy of higher-step inertia.
\subsubsection{Impact of \texorpdfstring{$\rho$}{rho}}
This experiment explores $\rho \in \{0.1, 0.2, \dots, 0.9\}$ with $(n, N, k_0) = (100, 100, 10)$, reporting medians over 20 trials in Figure~\ref{LinReg-rho}. As anticipated, a higher $\rho$ enhances efficiency for all methods, reducing CR and runtime.
Once again, \MA~achieves superior performance in almost all cases. Furthermore, \3A consistently outperforms \2A, which in turn exceeds \1A, underscoring the advantage of multi-step inertial technique.
\begin{figure}[htbp]
\centering
\includegraphics[width=.49\linewidth]{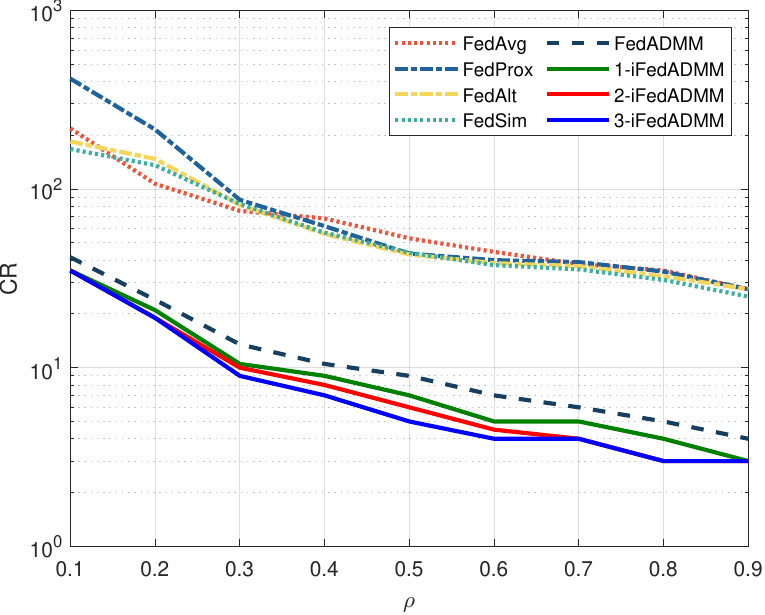}
\includegraphics[width=.49\linewidth]{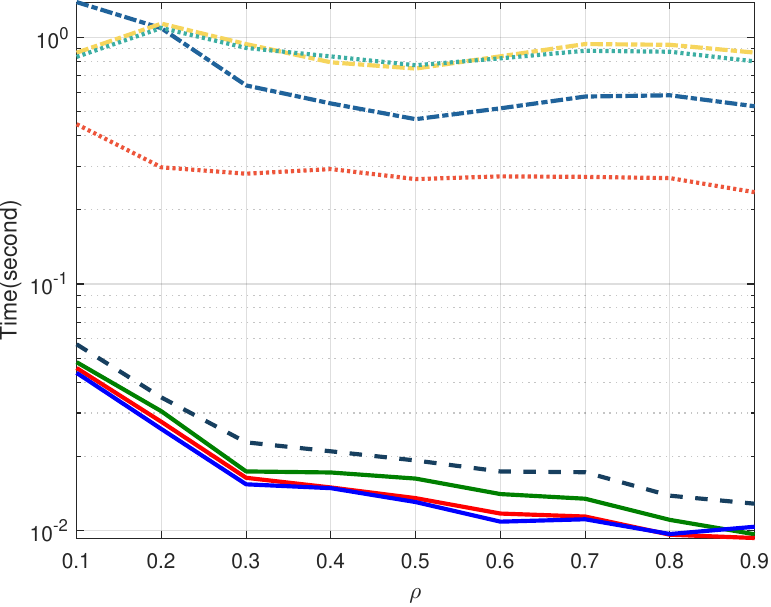}
\caption{Impact of $\rho$.\label{LinReg-rho}}
\end{figure}
\subsubsection{Impact of \texorpdfstring{$N$}{N}}
Similarly, we fix $(n, \rho, k_0) = (100, 0.5, 10)$ and vary $N \in \{50, 100, 150, 200\}$. Median results over 20 instances are shown in Figure~\ref{LinReg-N1}. It is evident that a larger $N$ (i.e., more clients) leads to fewer CR but longer computational time per round due to the increased workload. \MA~again proves most efficient, and the graded performance from \1A to \3A persists, reiterating the advantage of multi-step inertial scheme.
\begin{figure}[htbp]
\centering
\includegraphics[width=.49\linewidth]{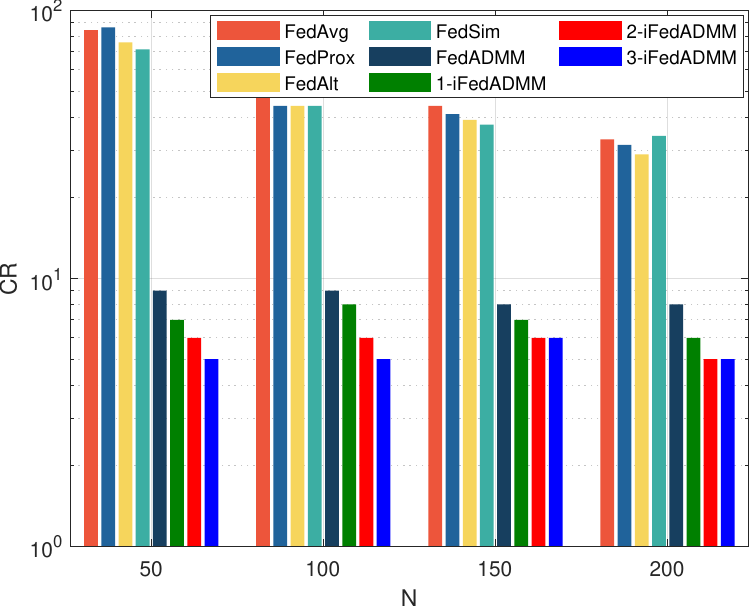}
\includegraphics[width=.49\linewidth]{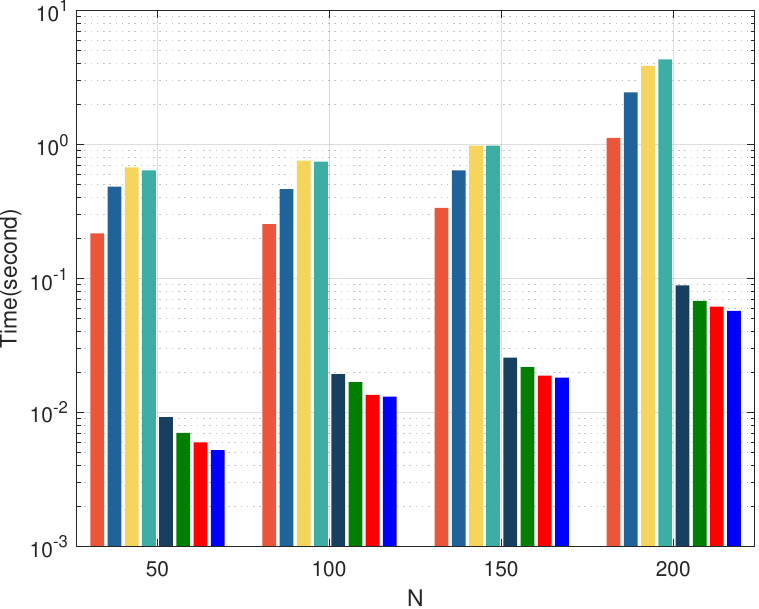}
\caption{Impact of $N$.\label{LinReg-N1}}
\end{figure}
\subsubsection{Impact of larger sizes}
We further evaluate scalability by fixing $(n, \rho, k_0) = (1000, 0.5, 10)$ and varying $N \in \{1000, 2000, 3000, 4000\}$.  Median results  from 20 trials are shown in Figure~\ref{LinReg-N2}. All algorithms exhibit a clear decrease in CR and an increase in computational time as $N$ grows. Notably, \MA~significantly outperforms other methods in runtime efficiency. Moreover, the performance order remains consistent: \3A surpasses \2A, which in turn exceeds \1A, further validating the advantage of multi-step inertial technique.
\begin{figure}[htbp]
\centering
\includegraphics[width=.49\linewidth]{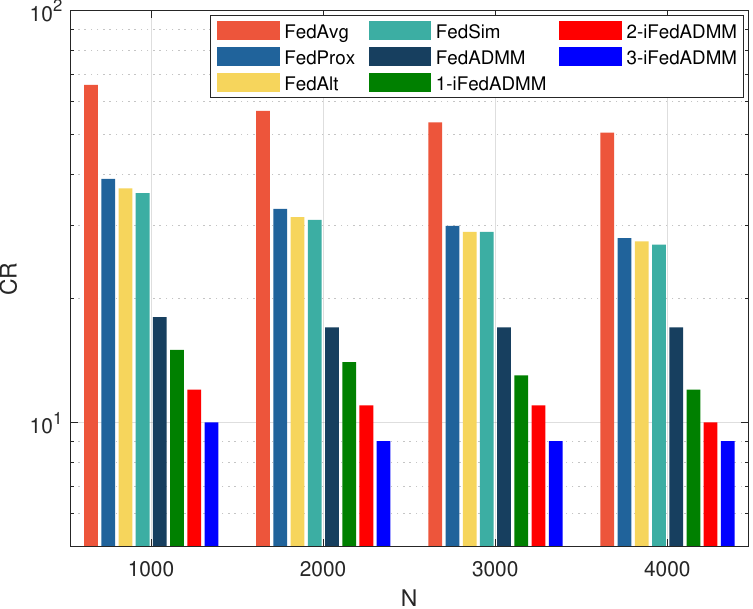}
\includegraphics[width=.49\linewidth]{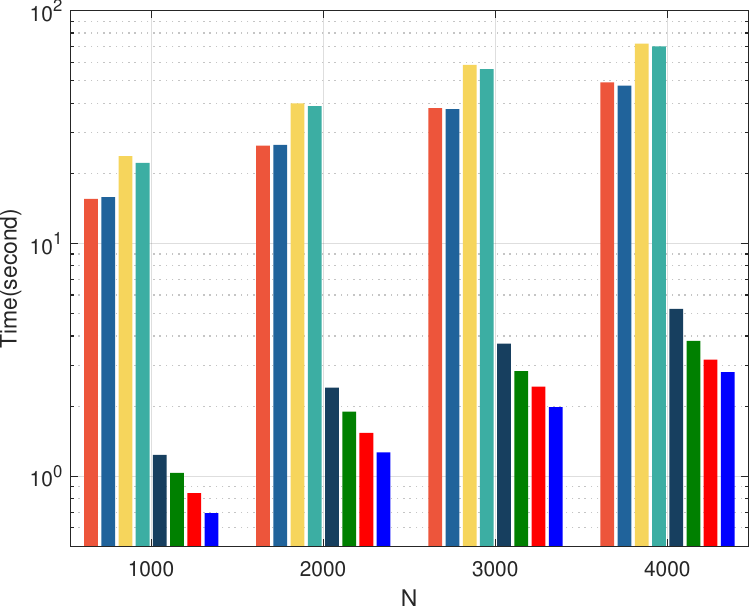}
\caption{Impact of larger sizes.\label{LinReg-N2}}
\end{figure}

\subsection{Logistic regression}\label{ex-logist}
In this experiment, each client $i$ has a local objective function for logistic regression with $\ell_2$-regularization:
$$
		f_i(\bx) = \frac{1}{d_i} \sum\limits_{t=1}^{d_i} \left[ \ln\left(1 + \exp(\langle \mathbf{a}_i^t, \bx \rangle)\right) - b_i^t \langle \mathbf{a}_i^t, \bx \rangle \right] + \frac{\lambda}{2} \| \bx \|^2,$$
where $\mathbf{a}_i^t \in \mathbb{R}^n$ is a feature vector, $b_i^t \in \{0,1\}$ is the corresponding binary label, and $\lambda = 0.001$ is a regularization parameter. We use two real-world datasets, summarized in Table~\ref{tab:datasets}, to generate the pairs $(\mathbf{a}, b)$. The samples are randomly divided among the $N$ clients.

\begin{table}[htbp]
	\renewcommand{\arraystretch}{1.25}
	\caption{Description of two real-world datasets.}
	\label{tab:datasets}
	\centering
	\begin{tabular}{ l l r r}
		\hline
		Dataset & Source & $n$ (features) & $d$ (samples) \\
		\hline
		{\tt Qsar Oral Toxicity} & UCI & 1024 & 8992 \\
		{\tt real-sim} & Libsvm & 20958 & 72309 \\
		\hline
	\end{tabular}
\end{table}
\subsubsection{Impact of \texorpdfstring{$k_0$}{k0}}
The effect of $k_0 \in \{1, 5, 10, \dots, 50\}$ is evaluated with $(N, \rho) = (100, 0.5)$ on the {\tt Qsar Oral Toxicity} dataset. Median results over 20 trials per $k_0$ value are shown in Figure~\ref{fig:effect-k0}. 
In contrast to \FA, \FP, \FAT, and \FS, whose CR decrease with $k_0$, \ADMM and \MA exhibit stable CR. Unsurprisingly, computational time rises with $k_0$ for all methods. Notably, \MA\ achieves the shortest runtime among all algorithms. Furthermore, the performance hierarchy remains consistent: \3A outperforms \2A, which in turn surpasses \1A, reinforcing the effectiveness of multi-step inertial technique.
\begin{figure}[htbp]
\centering
\includegraphics[width=.49\linewidth]{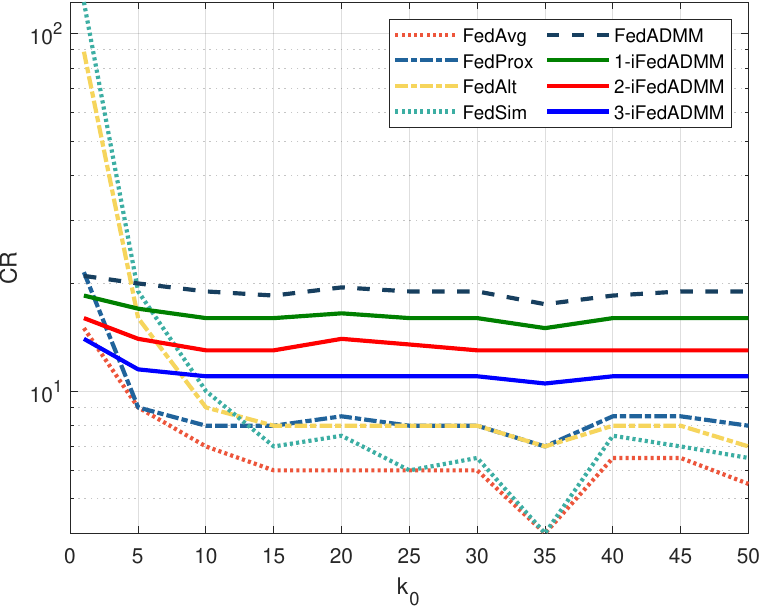}
\includegraphics[width=.49\linewidth]{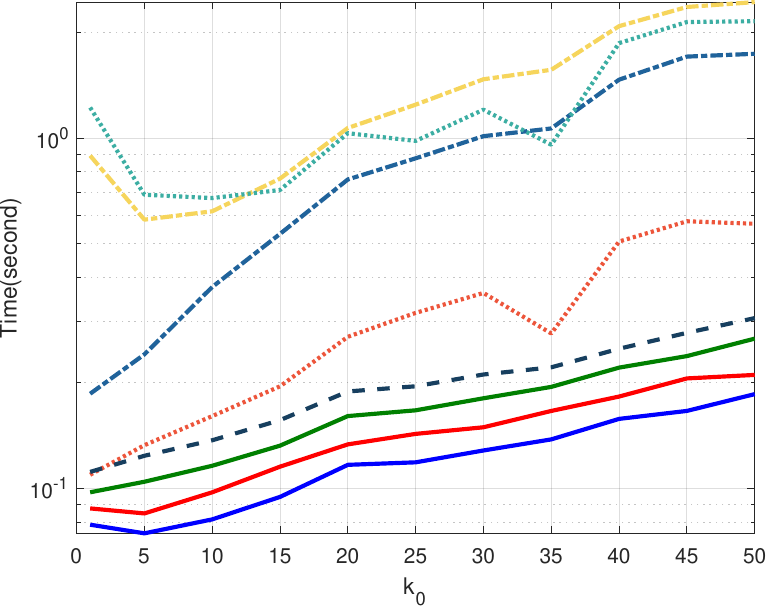}
\caption{Impact of $k_0$.\label{fig:effect-k0}}
\end{figure}

\subsubsection{Impact of \texorpdfstring{$\rho$}{rho}}
We investigate $\rho \in \{0.1, 0.2, \dots, 0.9\}$ with $(N, k_0) = (100, 10)$ using the {\tt Qsar Oral Toxicity} dataset.
Median results over 20 instances are presented in Figure~\ref{fig:effect-rho}. As $\rho$ increases, CR decreases for all algorithms except \FP, which maintains a stable CR count. 
Computational time trends diverge: it grows for \FP, \FAT, and \FS, but declines for others.
Consistently, \MA\ maintains the lowest runtime, and the consistent superiority of \3A over \2A over \1A underscores the efficacy of inertial steps.
\begin{figure}[htbp]
	\centering
	\includegraphics[width=.49\linewidth]{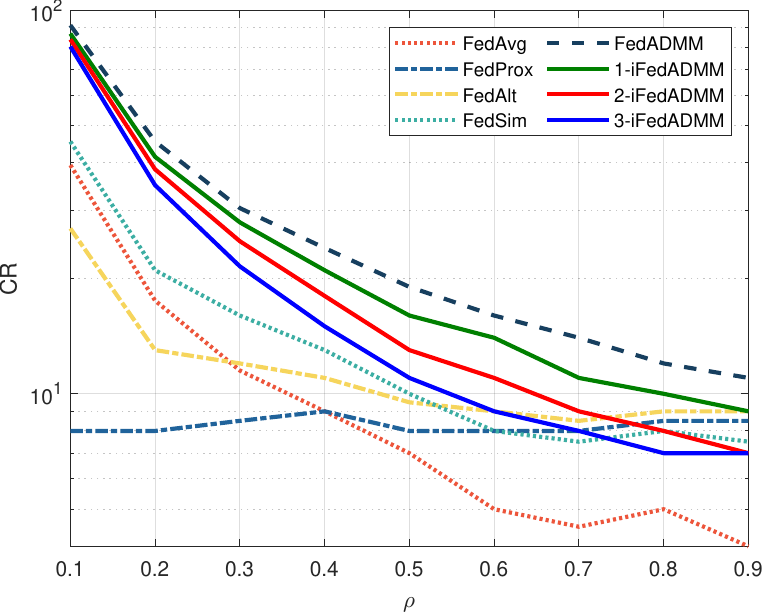}
	\includegraphics[width=.49\linewidth]{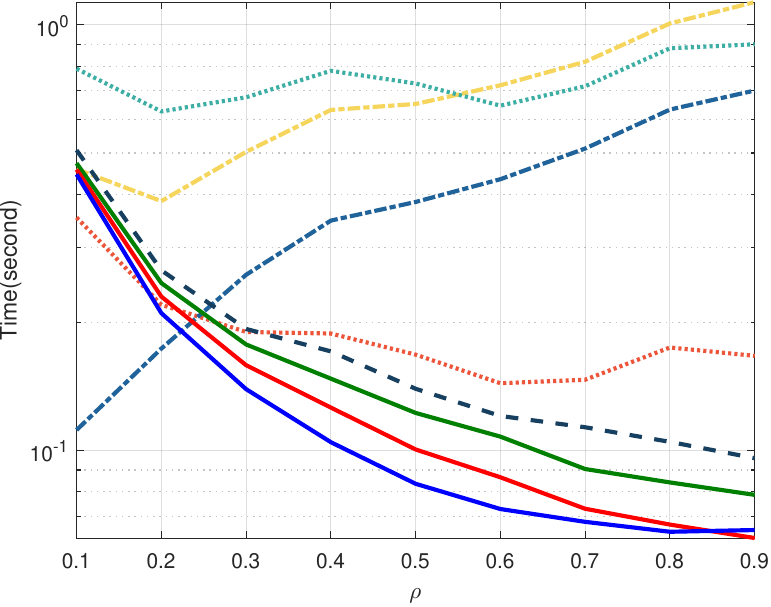}
	\caption{Impact of $\rho$.\label{fig:effect-rho}}
\end{figure}
\begin{figure}[htbp]
\centering
\includegraphics[width=.49\linewidth]{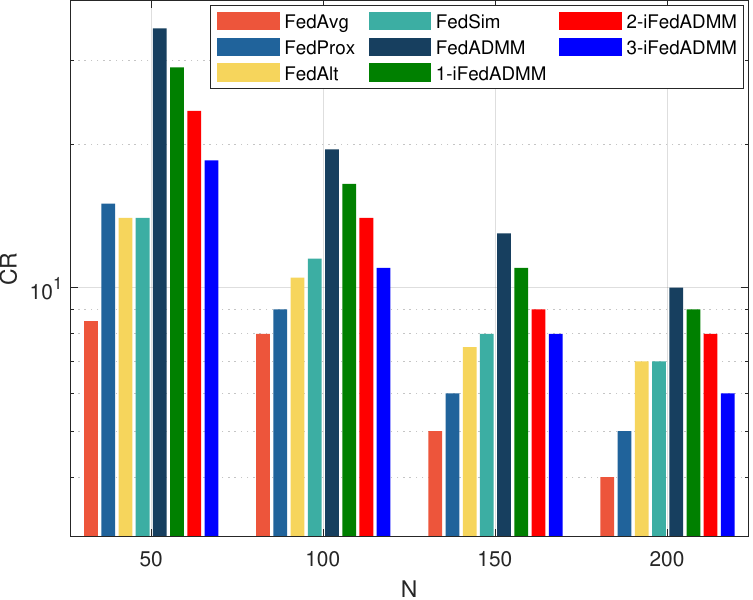}
\includegraphics[width=.49\linewidth]{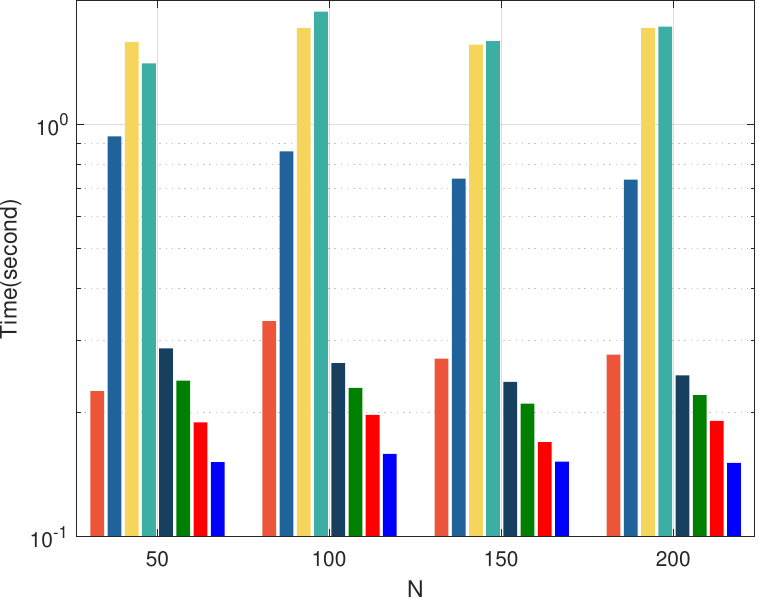}
\caption{Impact of $N$.\label{fig:effect-m}}
\end{figure}
\subsubsection{Impact of \texorpdfstring{$N$}{N}}
We assess scalability by varying $N \in \{50, 100, 150, 200\}$ on the {\tt Qsar Oral Toxicity} dataset, with $(\rho, k_0) = (0.5, 10)$. Figure~\ref{fig:effect-m} presents median results over 20 trials. As $N$ grows, all algorithms require fewer communication rounds (CR); the associated computational time increases moderately. \MA\ maintains the shortest runtime, and the consistent performance order—\3A leading \2A and \1A—again validates the multi-step inertial scheme.
\subsubsection{Impact of larger sizes}
We examine scalability on the high‑dimensional {\tt real‑sim dataset} ($n = 20,958$, $d = 72,309$) with $(\rho, k_0) = (0.5, 10)$ and $N \in  \{1000, 2000, 3000, 4000\}$. Median results over 20 trials are shown in Figure~\ref{fig:effect-hig}. As $N$ increases, both CR and computational time rise across all algorithms. \MA\ retains a decisive efficiency advantage, consistently achieving the shortest runtime. A clear performance hierarchy persists, with \3A leading, followed by \2A and then \1A, further underscoring the value of the multi‑step inertial scheme.
\begin{figure}[htbp]
\centering
\includegraphics[width=.49\linewidth]{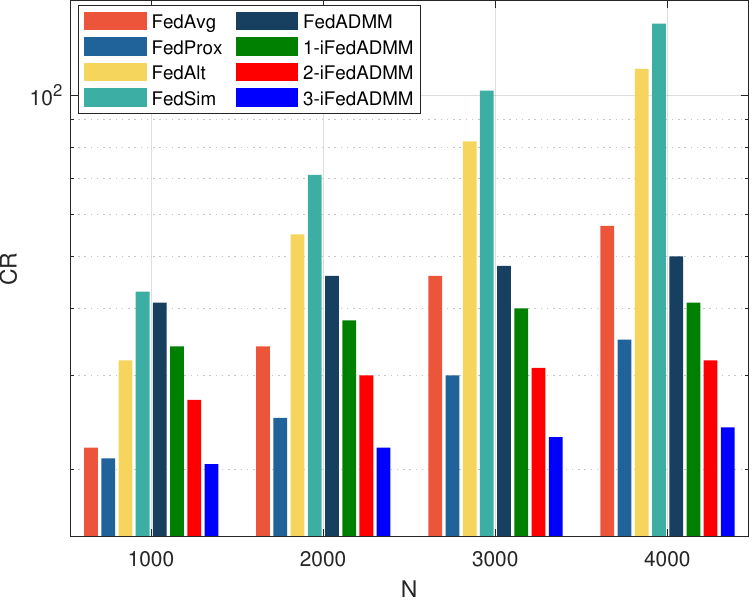}
\includegraphics[width=.49\linewidth]{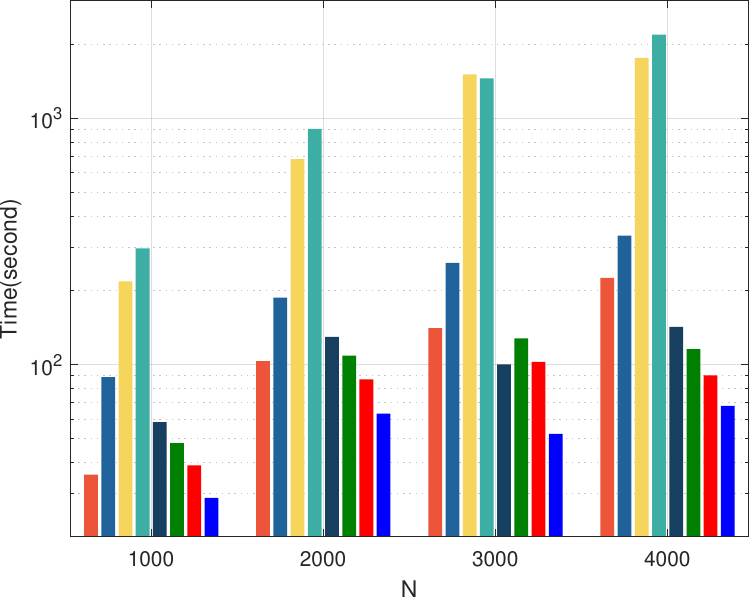}
\caption{Impact of larger sizes.\label{fig:effect-hig}}\vspace{-4mm}
\end{figure}

\section{Conclusion}\label{sec:conclusion}
We propose \MA, which extends \ADMM\ by incorporating a multi-step inertial technique during global updates to accelerate convergence and improve communication efficiency. 
For the single-step inertial variant, we establish global convergence to stationary points under mild assumptions and derive a non-ergodic $\mathcal{O}(1/K)$ convergence rate.
Numerical experiments on linear regression and logistic regression demonstrate that the proposed method significantly reduces the number of CRs required to achieve a desired accuracy, outperforming existing baselines.

To fully realize the potential of \MA, future research should address two frontiers. 
The first is to extend the convergence analysis from the single-step to the full multi-step inertial scheme, providing complete theoretical guarantees for \MA. 
The second involves incorporating SGD into the local updates by replacing full-batch gradients, thereby significantly reducing the per-round computational cost and further enhancing overall efficiency.
Establishing the convergence and convergence rate of this stochastic variant under non-i.i.d. data would be a valuable theoretical and practical contribution.

\newpage
\appendix
\section{Some basics}
For any two points $\bx_1$ and $\bx_2$, and for any $\bx_i \in \{\bx_1, \bx_2\}$, we have
\[
\bx_2 + t(\bx_1 - \bx_2) - \bx_i =
\begin{cases}
	(t-1)(\bx_1 - \bx_2) & \text{if } \bx_i = \bx_1, \\
	t(\bx_1 - \bx_2) & \text{if } \bx_i = \bx_2.
\end{cases}
\]
If the function is gradient Lipschitz continuous with constant $L$, then by the Mean Value Theorem, it follows that
\begin{eqnarray}\label{H-Lip-continuity-fxy}
	\arraycolsep=1.0pt\def\arraystretch{1.5}
	\begin{array}{llll}
		f(\bx_1) - f(\bx_2) -\langle\nabla  f(\bx_i), \bx_1-\bx_2\rangle 
		&=&  \int_0^1 \langle \nabla  f(\bx_2  +t(\bx_1-\bx_2  )) - \nabla  f(\bx_i ), \bx_1-\bx_2     \rangle dt  \\
		&\leq&  \int_0^1 L\|\bx_2  +t(\bx_1-\bx_2  ) -  \bx_i \|\| \bx_1-\bx_2     \| dt\\
		&=& \frac{L}{2}\| \bx_1-\bx_2\|^2.
	\end{array}
\end{eqnarray}
For any vectors $\bx_i$ and any $t > 0$, the following holds:
\begin{align}
	\|\bx_1-\bx_3\|^2  - \|\bx_2-\bx_3\|^2 =&  2\langle\bx_1-\bx_3,\bx_1-\bx_2\rangle
	-\|\bx_1-\bx_2\|^2,\label{two-vecs-1}\\
	2 \langle \bx_1, \bx_2\rangle \leq &  t\|\bx_1\|^2 + (1/t)\|\bx_2\|^2,\label{two-vecs-2}\\
	\|\bx_1+\bx_2\|^2 \leq & (1+t)\|\bx_1\|^2+(1+ 1/t)\|\bx_2\|^2,\label{two-vecs-3}\\
	\left\|\sum_{i=1}^N\bx_i\right\|^2 \leq & N\sum_{i=1}^N\|\bx_i\|^2.\label{two-vecs-4}
\end{align}

\section{Proofs of theorems in Section \ref{sec:convergence}}
\subsection{Proof of  Lemma \ref{basic-observations}}\label{lemma1}
\begin{proof} a) For each client $i\in[N]$ at $(k+1)$th iteration, let $k_i$ denote the largest integer in $[-1,k]$ such that $i\in\mathcal{N}^{\tau_{k_i+1}}$. This implies that client $i$ is not selected in any subsequent subset $\mathcal{N}^{\tau_{k_i+2}}, \mathcal{N}^{\tau_{k_i+3}} \cdots ,\tauk$. By update rule \eqref{iceadmm-sub5}, it follows that
	\begin{eqnarray}\label{par-invariance}
		\arraycolsep=0pt\def\arraystretch{1.5}
		\begin{array}{rcl}
			(\varepsilon_i^{\ell+1}, \by^{\ell+1}_i, \blam^{\ell+1}_i, \bz^{\ell+1}_i)\equiv(\varepsilon_i^{k_i+1}, \by^{k_i+1}_i,  \blam^{k_i+1}_i, \bz^{k_i+1}_i), \quad\forall \ell= k_i,k_i+1, \cdots ,k.
		\end{array}
	\end{eqnarray}
	For any client $i\in\tauk$, \eqref{iceadmm-sub4} holds directly. Now consider a client $i\notin\tauk$. If $k_i\geq0$, then since $i\in \mathcal{N}^{\tau_{k_i + 1}}$, $(\by^{k_i + 1}_i, \blam^{k_i + 1}_i, \bz^{k_i + 1}_i)$ satisfies \eqref{iceadmm-sub4}. By \eqref{par-invariance}, it follows that $(\by^{k + 1}_i, \blam^{k + 1}_i, \bz^{k + 1}_i)$ also satisfies \eqref{iceadmm-sub4}. If $k_i = -1$, indicating that client $i$ has never been selected, then by \eqref{par-invariance} and the initialization scheme, we conclude that
	$$\bz^{k+1}_i=\bz^{k_i+1}_i=\bz^{0}_i =\beta_i\by_i^{0} +\blam_i^{0}=\beta_i\by_i^{k+1} +\blam_i^{k+1}.$$
	Therefore, \eqref{iceadmm-sub4} remains valid in all cases. We thus conclude that \eqref{iceadmm-sub4} holds for every client $i \in [N]$ and all iterations $k \geq 0$. Now, for any $k \in \mathcal{K}$,
	\begin{eqnarray*}
		\arraycolsep=1.0pt\def\arraystretch{1.5}
		\begin{array}{lcl}
			\sum_i  (  \beta_i(\by_i^k-\hat\bx^{\tk})+\blam_i^k)\overset{\eqref{iceadmm-sub4}}= \sum_i ( \bz_i^k -\beta_i\hat\bx^{\tk})   \overset{\eqref{iceadmm-sub1}}{=} 0.
		\end{array}
	\end{eqnarray*}
    
	b) For any  client $i\in\tauk$,  the solution $\by_i^{k+1}$ obtained from \eqref{iceadmm-sub2}  satisfies
	\begin{align*}
		\Phi_i^{k+1} = \bg_i^{k+1}  +  \blam_i^{k+1}\overset{\eqref{iceadmm-sub3}}{=}& \bg_i^{k+1} + \blam_i^k  +\beta_i \triangle\oby_i^{k+1},\\
		\|\Phi_i^{k+1} \|^2\overset{\eqref{iceadmm-sub2}}{\leq}&\varepsilon_i^{k+1}.
	\end{align*}
	For any client $i\notin\tauk$, it holds that
	$$\Phi_i^{k_i+1}=\bg_i^{k_i+1}  +  \blam_i^{k_i+1},~~\|\Phi_i^{k_i+1} \|^2\leq\varepsilon_i^{k_i+1}$$
	since  $i\in\mathcal{N}^{\tau_{k_i+1}}$. Combined with \eqref{par-invariance}, this ensures that \eqref{opt-con-xk1-2} continues to hold. Therefore, we conclude that \eqref{opt-con-xk1-2} is satisfied for every client $i \in [N]$ and all iterations $k \geq 0$.

	c) For $i\in\tauk$, Assumption \ref{ass-fi} and \eqref{opt-con-xk1-2} imply that
	\begin{eqnarray} \label{gap-pi-k-eps}
		\arraycolsep=1.0pt\def\arraystretch{1.5}
		\begin{array}{lcl}
			\|\triangle   \blam_i^{k+1}\|^2 &=& \| \bg_i^{k+1} - \bg_i^k -(\Phi_i^{k+1}-\Phi_i^k)\|^2\\
			&\overset{\eqref{two-vecs-3}}\leq& \frac{6p_i^2L_i^2}{5}  \|\triangle\by_i^{k+1} \|^2+6\| \Phi_i^{k+1}-\Phi_i^k\|^2\\
			&\overset{\eqref{opt-con-xk1-2}}\leq& \frac{6p_i^2L_i^2}{5}  \|\triangle\by_i^{k+1} \|^2+12(\varepsilon_i^{k+1}+\varepsilon_i^k)\\
			&\overset{\eqref{iceadmm-sub20}}=&\frac{6p_i^2L_i^2}{5} \|\triangle\by_i^{k+1} \|^2+\frac{12(1+\theta_{i})}{1-\theta_{i}}(\varepsilon_i^k-\varepsilon_i^{k+1})\\
			&\leq&\frac{6p_i^2L_i^2}{5} \|\triangle\by_i^{k+1} \|^2-\frac{24}{1-\theta_{i}}\triangle\varepsilon_i^{k+1}.
		\end{array}
	\end{eqnarray}
	For $i\notin\tauk$, \eqref{iceadmm-sub5} gives $\triangle \blam_i^{k+1} = 0$, so the condition remains satisfied.
	
	d) For sufficiently large $k \in \mathcal{K}$, Scheme \ref{ass-omega} ensures that every client $i$ has been selected at least $k/(s_0k_0)$ times. Consequently, \eqref{iceadmm-sub20} is invoked at least $k/(s_0k_0)$ times, ensuring $\varepsilon_i^{k+1} \to 0$.
\end{proof}
\subsection{Proof of Lemma \ref{lemma-decreasing-1}}\label{lemma2}
\begin{proof}
	We bound the gap $(\L^{k+1}-\L^k)$ by decomposing it into three components:
	\begin{align}\label{three-cases}
		\L^{k+1}-\L^k  =: p_1^k+p_2^k+p_3^k,
	\end{align}
	where
	\begin{eqnarray*}
		\arraycolsep=1.0pt\def\arraystretch{1.5}
		\begin{array}{llllll}
			p_1^k&:=&\L(\bx^{\tk},Y^k,\Lambda^k)-\L(\bx^{\tau_k},Y^k,\Lambda^k),\\
			p_2^k&:=& \L(\bx^{\tk},Y^{k+1},\Lambda^k)-\L(\bx^{\tk},Y^k,\Lambda^k), \\
			p_3^k&:=& \L(\bx^{\tk},Y^{k+1},\Lambda^{k+1})-\L(\bx^{\tk},Y^{k+1},\Lambda^k).
		\end{array}
	\end{eqnarray*}
	Since $\beta_i\geq 3p_i L_i > 0$, we have
	\begin{eqnarray}\label{sigma-a-r}
		\arraycolsep=1.0pt\def\arraystretch{1.5}
		\begin{array}{llllll}
			0< p_i L_i\leq  \frac{\beta_i}{3}, ~~\forall i\in[N].
		\end{array}
	\end{eqnarray}
    
	{\underline{Bound  $p_1^k$.}}
	If  $k\in\mathcal{K}$, then $\tk=\tau_{k}+1$, and
	\begin{eqnarray*}
		\arraycolsep=1.0pt\def\arraystretch{1.5}
		\begin{array}{lcl}
			p_1^k
			&\overset{\eqref{Def-L}}=& \sum_i\langle\triangle\bx^{\tk}, -\blam_i^k\rangle+\sum_i (  \frac{\beta _i}{2}\|\by_i^k-\bx^{\tk}\|^2 -\frac{\beta _i}{2}\|\by_i^k-\bx^{\tau_k}\|^2 )\\
			&\overset{\eqref{two-vecs-1}}=& \sum_i    \langle\triangle\bx^{\tk}, -\blam_i^k-\beta _i(\by_i^k-\bx^{\tk})\rangle- \sum_i \frac{\beta_i}{2} \| \triangle\bx^{\tk}\|^2 \\
			&\overset{\eqref{opt-con-xk1-1}}=&\sum_i\langle\triangle\bx^{\tk}, \beta _i(\bx^{\tk}-\hat\bx^{\tk})\rangle
			- \sum_i \frac{\beta_i}{2} \| \triangle\bx^{\tk}\|^2 \\
			&=&\frac{\beta}{2}\left(\|\bx^{\tk}-\hat{\bx}^{\tk}\|^2-\|\bx^{\tau_k}-\hat{\bx}^{\tk}\|^2\right)\\
			&\overset{\eqref{accelerated}}=&\frac{\beta}{2}(a_{\tk}^2\|\hat{\bx}^{\tk}-\hat{\bx}^{\tk-1}\|^2-\|a_{\tau_k}(\hat{\bx}^{\tau_k}-\hat{\bx}^{\tau_k-1})-(\hat{\bx}^{\tk}-\hat{\bx}^{\tk-1})\|^2)\\
			&=&\frac{\beta}{2}((a_{\tk}^2-1)\|\hat{\bx}^{\tk}-\hat{\bx}^{\tk-1}\|^2
			-a_{\tau_k}^2\|\hat{\bx}^{\tau_k}-\hat{\bx}^{\tau_k-1}\|^2+2a_{\tau_k}\langle\hat{\bx}^{\tk}-\hat{\bx}^{\tk-1},\hat{\bx}^{\tau_k}-\hat{\bx}^{\tau_k-1}\rangle)\\
			&\overset{\eqref{two-vecs-2}}\leq&\frac{\beta}{2}((a_{\tk}^2-1)\|\hat{\bx}^{\tk}-\hat{\bx}^{\tk-1}\|^2-a_{\tau_k}^2\|\hat{\bx}^{\tau_k}-\hat{\bx}^{\tau_k-1}\|^2\\
            &&+a_{\tau_k}(\|\hat{\bx}^{\tk}-\hat{\bx}^{\tk-1}\|^2+\|\hat{\bx}^{\tau_k}-\hat{\bx}^{\tau_k-1}\|^2))\\
			&=&\frac{\beta}{2}((a_{\tk}^2+a_{\tau_k}-1)\|\hat{\bx}^{\tk}-\hat{\bx}^{\tk-1}\|^2+(a_{\tau_k}-a_{\tau_k}^2)\|\hat{\bx}^{\tau_k}-\hat{\bx}^{\tau_k-1}\|^2)
		\end{array}
	\end{eqnarray*}
	Adding $\frac{\beta}{2}((a_{\tk}-a_{\tk}^2)\|\hat{\bx}^{\tk}-\hat{\bx}^{\tk-1}\|^2-(a_{\tau_k}-a_{\tau_k}^2)\|\hat{\bx}^{\tau_k}-\hat{\bx}^{\tau_k-1}\|^2)$ to both sides of the inequality yields
	\begin{eqnarray}\label{p1}
		\arraycolsep=1.0pt\def\arraystretch{1.5}
		\begin{array}{lcl}
			&&p_1^k+\frac{\beta}{2}((a_{\tk}-a_{\tk}^2)\|\hat{\bx}^{\tk}-\hat{\bx}^{\tk-1}\|^2-(a_{\tau_k}-a_{\tau_k}^2)\|\hat{\bx}^{\tau_k}-\hat{\bx}^{\tau_k-1}\|^2)\\
			&\leq&\frac{\beta}{2}(a_{\tk}+a_{\tau_k}-1)\|\hat{\bx}^{\tk}-\hat{\bx}^{\tk-1}\|^2\\
			&=&\frac{\beta}{2}(a_{\tk}+a_{\tau_k}-1)\|\triangle\hat{\bx}^{\tk}\|^2.
		\end{array}
	\end{eqnarray}
where the last equality follows from $\tk=\tau_{k}+1$.
	
If  $k\notin\mathcal{K}$, $\tk=\tau_k$. In this case, \eqref{p1} continues to hold.

{\underline{Bound  $p_2^k$.}}  We proceed by considering two cases: $i\in\tauk$ and $i\notin\tauk$. For $i\in\tauk$, define $u_i$ as  follows:
\begin{eqnarray*}
\arraycolsep=1.0pt\def\arraystretch{1.5}
\begin{array}{lcl}
			u_i^k:
			&=&     \L_i(\bx^{\tk},\by_i^{k+1},\blam_i^k)-\L_i(\bx^{\tk},\by_i^k,\blam_i^k) \\
			&\overset{\eqref{Def-L}}=&      p_i  f_i(\by_i^{k+1})- p_i f_i(\by_i^k)   +  \langle \triangle \by_i^{k+1} , \blam_i^k\rangle+        \frac{\beta _i}{2}\|\triangle\oby_i^{k+1}\|^2 -\frac{\beta _i}{2}\|\by_i^k-\bx^{\tk}\|^2 \\
			&\overset{\eqref{two-vecs-1}}=& p_i  f_i(\by_i^{k+1})- p_i f_i(\by_i^k)+   \langle \triangle \by_i^{k+1} , \blam_i^k+\beta_i \triangle\oby_i^{k+1}\rangle-  \frac{\beta_i}{2} \| \triangle\by_i^{k+1}\|^2\\
			&\overset{\eqref{iceadmm-sub3} }=&      p_i  f_i(\by_i^{k+1})- p_i f_i(\by_i^k)  +   \langle \triangle \by_i^{k+1} , \blam_i^{k+1}\rangle-  \frac{\beta_i}{2} \| \triangle\by_i^{k+1}\|^2\\
			&\overset{\eqref{H-Lip-continuity-fxy}}\leq& \frac{p_iL_i-\beta_i}{2}\| \triangle\by_i^{k+1}\|^2 +   \langle \triangle\by_i^{k+1}, \bg_i^{k+1} + \blam_i^{k+1}\rangle   \\
			&\overset{\eqref{sigma-a-r} }\leq&- \frac{\beta_i}{3}\| \triangle\by_i^{k+1}\|^2 +   \langle \triangle\by_i^{k+1}, \bg_i^{k+1} + \blam_i^{k+1}\rangle    \\
			&\overset{ \eqref{opt-con-xk1-2} }=&- \frac{\beta_i}{3}\| \triangle\by_i^{k+1}\|^2 + \langle \triangle\by_i^{k+1}, \Phi_i^{k+1}\rangle     \\
			&\overset{\eqref{two-vecs-2}}\leq& -\frac{\beta _i}{3}\| \triangle\by_i^{k+1}\|^2+ \frac{5}{\beta_i}\|\Phi_i^{k+1}\|^2 +\frac{\beta _i}{10}\| \triangle\by_i^{k+1}\|^2  \\
			&\overset{\eqref{opt-con-xk1-2}}\leq&-\frac{7\beta _i}{30}\| \triangle\by_i^{k+1}\|^2+  \frac{5\varepsilon_i^{k+1}}{\beta_i} \\
			&\overset{\eqref{iceadmm-sub20}}\leq& -\frac{7\beta _i}{30}\| \triangle\by_i^{k+1}\|^2 - \frac{5\triangle\varepsilon_i^{k+1}}{(1-\theta_{i})\beta_i}.
		\end{array}
	\end{eqnarray*}
	For $i\notin\tauk$,   $ \by_i^{k+1}=\by_i^k$ and $\varepsilon_i^{k+1}=\varepsilon_i^k$ from \eqref{iceadmm-sub5} imply that the above condition again holds. Therefore, in both cases, we obtain
	\begin{eqnarray}\label{gap-2-1}
		\arraycolsep=0.5pt\def\arraystretch{1.5}
		\begin{array}{lcl}
			p_{2}^k = \sum_i  u_i^k  \leq \sum_i (-\frac{7\beta _i}{30}\| \triangle\by_i^{k+1}\|^2 - \frac{5\triangle\varepsilon_i^{k+1}}{(1-\theta_{i})\beta_i}) .
		\end{array}
	\end{eqnarray}
    
	{\underline{Bound  $p_3^k$.}} We consider two cases: $i\in\tauk$ and $i\notin\tauk$. If $i\in\tauk$, then clearly
	\begin{eqnarray*}
		\arraycolsep=1.0pt\def\arraystretch{1.5}
		\begin{array}{lcl}
			v_i^k :
			&=&     \L_i(\bx^{\tk},\by_i^{k+1},\blam_i^{k+1})-\L_i(\bx^{\tk},\by_i^{k+1},\blam_i^k) \\
			&\overset{\eqref{Def-L}}=&     \langle \triangle \oby_i^{k+1} ,\triangle \blam_i^{k+1}  \rangle
			\overset{\eqref{iceadmm-sub3}}=    \frac{1}{\beta _i}\|\triangle \blam_i^{k+1} \|^2 \\
			&\overset{\eqref{opt-con-xk1-3}}\leq&   \frac{6p_i^2L_i^2}{5\beta _i}  \|\triangle\by_i^{k+1} \|^2- \frac{24}{(1-\theta_{i})\beta_i} \triangle\varepsilon_i^{k+1}\\
			&\overset{\eqref{sigma-a-r}}\leq&   \frac{4\beta_i}{30}  \|\triangle\by_i^{k+1} \|^2- \frac{24}{(1-\theta_{i})\beta_i} \triangle\varepsilon_i^{k+1}.
		\end{array}
	\end{eqnarray*}
For $i \notin \mathcal{N}^{\tau_k}$, $\by_i^{k+1}{=}\by_i^k$ and $\varepsilon_i^{k+1}{=}\varepsilon_i^k$ from \eqref{iceadmm-sub5} ensure the final inequality continues to hold. Hence, in both cases we conclude that
\begin{eqnarray}\label{gap-3-1}
\arraycolsep=0.5pt\def\arraystretch{1.5}
\begin{array}{lcl}
p_3^k =  \sum_i v_i^k   \leq    \sum_i  ( \frac{4\beta_i}{30}  \|\triangle\by_i^{k+1} \|^2- \frac{24\triangle\varepsilon_i^{k+1}}{(1-\theta_{i})\beta_i} ).
\end{array}
\end{eqnarray}
Combining \eqref{three-cases}, \eqref{p1}, \eqref{gap-2-1}, and \eqref{gap-3-1} yields
\begin{eqnarray*}
		\arraycolsep=1.0pt\def\arraystretch{1.5}
		\begin{array}{lll}
			\L^{k+1}-\L^k
			&\leq&\sum_i\left(-\frac{\beta _i}{10}\|\triangle\by_i^{k+1}\|^2-\frac{29\triangle\varepsilon_i^{k+1}}{(1-\theta_{i})\beta_i}\right)-\frac{\beta}{2}(1-a_{\tk}-a_{\tau_k})\|\triangle\hat{\bx}^{\tk}\|^2\\
			&&-\frac{\beta}{2}((a_{\tk}-a_{\tk}^2)\|\hat{\bx}^{\tk}-\hat{\bx}^{\tk-1}\|^2-(a_{\tau_k}-a_{\tau_k}^2)\|\hat{\bx}^{\tau_k}-\hat{\bx}^{\tau_k-1}\|^2)
		\end{array}
	\end{eqnarray*}
	The desired result follows:
	\begin{eqnarray*}
		\arraycolsep=1.0pt\def\arraystretch{1.5}
		\begin{array}{lll}
			\widetilde{\L}^{k+1}-\widetilde{\L}^k
			&\leq&-\sum_i\frac{\beta _i}{10}\|\triangle\by_i^{k+1}\|^2-\frac{\beta}{2}(1-a_{\tk}-a_{\tau_k})\|\triangle\hat{\bx}^{\tk}\|^2
		\end{array}
	\end{eqnarray*}
	where $\widetilde{\L}^k = \L^k + \sum_i\frac{29 \varepsilon_i^k }{(1-\theta_{i})\beta_i} + \frac{\beta}{2}(a_{\tau_k}-a_{\tau_k}^2)\|\hat{\bx}^{\tau_k} - \hat{\bx}^{\tau_k-1}\|^2$.
\end{proof}

\subsection{Proof of Lemma \ref{L-bounded-decreasing-1}}\label{lemma3}
\begin{proof}
	a) This claim follows immediately from Lemma~\ref{lemma-decreasing-1}.
	
	b) Assumption \ref{ass-fi} yields
	\begin{eqnarray*}
		\arraycolsep=1pt\def\arraystretch{1.5}
		\begin{array}{ccll}
			p_if_i(\bx^{\tau_k})- p_if_i(\by_i^k)
			&\overset{\eqref{H-Lip-continuity-fxy}}{\leq}& \langle \triangle \oby_i^k,  -\bg_i^k \  \rangle    + \frac{L_ip_i}{2}\|\triangle \oby_i^k\|^2 \\
			&\overset{\eqref{sigma-a-r}}{\leq}& \langle \triangle \oby_i^k,  -\bg_i^k \  \rangle    + \frac{\beta_{i}}{6}\|\triangle \oby_i^k\|^2 \\
			&\overset{\eqref{opt-con-xk1-2}}{=}&  \langle \triangle \oby_i^k,  \blam_i^k -\Phi_i^k  \rangle    + \frac{\beta_{i}}{6} \|\triangle \oby_i^k\|^2\\
			&\overset{\eqref{two-vecs-2}}{\leq}&  \langle \triangle \oby_i^k,  \blam_i^k  \rangle    + \frac{\beta_{i}}{3}\|\triangle \oby_i^k\|^2+\frac{3}{2\beta_i}\|\Phi_i^k\|^2 + \frac{\beta_{i}}{6} \|\triangle \oby_i^k\|^2 \\
			&\overset{\eqref{opt-con-xk1-2}}{\leq}&  \langle \triangle \oby_i^k,  \blam_i^k  \rangle    +\frac{\beta_{i}}{2}\|\triangle \oby_i^k\|^2+\frac{3\varepsilon_i^k}{2\beta_i}.
		\end{array}
	\end{eqnarray*}
	Combining this with $1 > \theta_i \geq 1/2$ yields
	\begin{eqnarray*}
		\arraycolsep=1.0pt\def\arraystretch{1.5}
		\begin{array}{ccll}
			\widetilde \L^k
			&\overset{\eqref{def-c}}=&  \sum_i   (\L_i(\bx^{\tau_k},\by_i^k,\blam_i^k)+\frac{29\varepsilon_i^k}{(1-\theta_{i})\beta_i})+\frac{\beta}{2}(a_{\tau_k}-a_{\tau_k}^2)\|\hat{\bx}^{\tau_k}-\hat{\bx}^{\tau_k-1}\|^2\\
			&\overset{\eqref{Def-L}}{\geq}&  \sum_i (p_i  f_i(\by_i^k) + \langle \triangle \oby_i^k, \blam_i^k\rangle + \frac{\beta _i}{2}\|\triangle \oby_i^k\|^2+\frac{3\varepsilon_i^k}{2\beta_i})\\
			&\geq& \sum_i  p_if_i(\bx^{\tau_k}) = f(\bx^{\tau_k})  \geq f(\bx^*),
		\end{array}
	\end{eqnarray*}
	Since each $f_i$ is bounded below, the global objective satisfies $f(\bx^*) > -\infty$, and thus $\widetilde{\mathcal{L}}^k \geq f(\bx^{\tau_k}) \geq f(\bx^*) > -\infty$.
	
	c) From Lemma \ref{lemma-decreasing-1} and $\widetilde{\mathcal{L}}^k > -\infty$, it follows that
	\begin{eqnarray*}
		\arraycolsep=0pt\def\arraystretch{1.5}
		\begin{array}{lll}
			\sum\limits_{k\geq0}\!\left(\sum_i\!\frac{\beta _i}{10}\|\triangle\by_i^{k+1}\|^2
			\!+\!\frac{\beta}{2}(1\!-\!a_{\scriptscriptstyle\tk}\!-\!a_{\scriptscriptstyle\tau_k})\|\triangle\hat{\bx}^{\scriptscriptstyle\tk}\|^2\right)
			\leq\sum\limits_{k\geq0}\! (\widetilde\L^k\!-\!\widetilde\L^{k+1})\leq\widetilde\L^{0}\!-\! f(\bx^*) < +\infty.
		\end{array}
	\end{eqnarray*}
	This implies $\| \triangle \hat{\bx}^{\tk}\|\to0$ and $\|\triangle\by_i^{k+1}\|\to0$. Moreover, from \eqref{opt-con-xk1-3} and \eqref{opt-con-xk1-4}, $\|\triangle\by_i^{k+1}\|\to0$ yields $\| \triangle \blam_i^{k+1} \| \to 0$. Finally, since $\| \triangle \hat{\bx}^{\tk}\|\to0$, we conclude $\|\triangle\obx^{\tk}\|\to 0$.
	
	Given $\| \triangle \hat{\bx}^{\tk}\|\to0$ and $\|\triangle\obx^{\tk}\|\to 0$, we have
	\begin{align*}
		\| \triangle\bx^{\tk}\|=\|-\triangle\obx^{\tk}+\triangle \hat{\bx}^{\tk}-\triangle\obx^{\tau_k}\|
		\leq\|\triangle\obx^{\tk}\|+\|\triangle \hat{\bx}^{\tk}\|+\|\triangle\obx^{\tau_k}\|\to 0
	\end{align*}
	Finally, for  $i\in\tauk$, \eqref{iceadmm-sub3} gives $\|\beta _i \triangle\oby_i^{k+1}\| = \|  \triangle \blam_i^{k+1} \| \to 0$.
	For $i\notin\tauk$, we have $i\in\mathcal{N}^{\tau_{k_i+1}}$ with $k_i$ defined as in the proof of Lemma \ref{basic-observations} a). By Scheme \ref{ass-omega}, $\tk -\tau_{k_i+1}\leq s_0$, and thus
	\begin{eqnarray}\label{gap-xi-xtau}
		\arraycolsep=1.0pt\def\arraystretch{1.5}
		\begin{array}{lcl}
			\| \triangle\oby_i^{k+1}  \|^2 &\overset{\eqref{par-invariance}}{=}&\| \by_i^{k_i+1}-\bx^{\tk}  \|^2 \\
			&\overset{\eqref{two-vecs-3}}\leq& 2\| \triangle\oby_i^{k_i+1}\|^2+2\|\bx^{\tau_{k_i+1}}-\bx^{\tk}  \|^2 \\
			&\overset{\eqref{iceadmm-sub3}}{=}& \frac{2}{\beta_i^2}\| \triangle \blam_i^{k_i+1} \|^2+2\|\sum\limits_{t= \tau_{k_i+1}}^{\tk-1}(\bx^{ {t+1}}-\bx^{ t})  \|^2 \\
			&\overset{\eqref{two-vecs-4}}{\leq}&
			\frac{2}{\beta_i^2}\| \triangle \blam_i^{k_i+1} \|^2+2 s_0\sum\limits_{t= \tau_{k_i+1}}^{\tk-1}\|\bx^{ {t+1}}-\bx^{ t}\|^2\to 0,
		\end{array}
	\end{eqnarray}
	where the last relation follows from $\triangle\blam_i^k\to0$ and $\triangle\bx^{\tk}\to0$. This completes the proof.
\end{proof}

\subsection{Proof of  Theorem \ref{global-obj-convergence-inexact}}\label{theorem1}
\begin{proof}
	a) By Lemma~\ref{L-bounded-decreasing-1}, $\widetilde\L^1\geq f(\bx^{\tk})$ holds. This together with the coercivity of $f$ implies the boundedness of the sequence $\{\bx^{\tk}\}$. Moreover, since $\triangle\oby_i^{k+1}\to0$ by \eqref{limit-5-term-0-inexact}, the sequence $\{\by_i^{k+1}\}$ is also bounded, thereby delivering
	\begin{eqnarray*}
		\arraycolsep=0.0pt\def\arraystretch{1.5}
		\begin{array}{lcl}
			\|\blam_i^{k+1}\|  &\overset{\eqref{opt-con-xk1-2}}{=}&   \|\Phi_i^{k+1}-\bg_i^{k+1}\| \\
			&\leq& \|\Phi_i^{k+1}\| +  \|\bg_i^{k+1}\!-\!\bg_i^{0}\|+ \|\bg_i^{0}\|  \\
			&\overset{\eqref{opt-con-xk1-2},\eqref{Lip-r}}{\leq}&\sqrt{\varepsilon_i^{k+1}}\!+\! p_i L_i \|\by_i^{k+1}\!-\!\by_i^0\|+ \|\bg_i^{0}\| <+\infty.
		\end{array}
	\end{eqnarray*}
	This establishes the boundedness of $\{\blam_i^{k+1}\}$. We thus conclude that $\{(\bx^{\tk},Y^{k+1},\Lambda^{k+1})\}$ is bounded.
	
	b) From Lemma~\ref{L-bounded-decreasing-1}, the sequence $\{\widetilde{\mathcal{L}}^k\}$ is non-increasing and bounded below, hence it converges. Moreover, $\widetilde{\mathcal{L}}^{k+1} \to \mathcal{L}^{k+1}$ because $\varepsilon_i^{k+1} \to 0$ and $\triangle \hat\bx^{\tau_k} \to 0$ by \eqref{limit-5-term-0-inexact}. Finally, using $\triangle \by_i^{k+1} \to 0$ from \eqref{limit-5-term-0-inexact} and the boundedness of $\{\blam_i^{k+1}\}$, we obtain
	\begin{eqnarray*}%\label{L-f-pi}
		\arraycolsep=0.0pt\def\arraystretch{1.5}
		\begin{array}{lcl}
			\L^{k+1}- F(Y^{k+1})
			\overset{\eqref{Def-L}}{=}  \sum_i  \left(  \langle \triangle \oby_i^{k+1}, \blam_i^{k+1}\rangle +\frac{\beta _i}{2}\|\triangle \oby_i^{k+1}\|^2 \right)  \to0.
		\end{array}
	\end{eqnarray*}
	By the Mean Value Theorem,
	$$f_i(\by_i^{k+1})=f_i(\bx^{\tk})+\langle\triangle \oby^{k+1} , \nabla f_i(\bx_t)\rangle,$$
	where $\bx_t := (1-t)\bx^{\tk}+t\by_i^{k+1}$ for some $t\in(0,1)$. The boundedness of $\{\bx^{\tk}, \by_i^{k+1}\}$ implies that $\bx_t$ is also bounded.
	\begin{eqnarray*}
		\arraycolsep=0pt\def\arraystretch{1.5}
		\begin{array}{lcl}
			\|\nabla f_i(\bx_t)\|\leq\|\nabla f_i(\bx_t)-\nabla f_i(\by_i^0)\|+ \|\nabla f_i(\by_i^0)\|
			\overset{\eqref{Lip-r}}\leq L_i\|\bx_t-\by_i^0\|+ \|\nabla f_i(\by_i^0)\| <+\infty.
		\end{array}
	\end{eqnarray*}
	The boundedness of $\nabla f_i(\bx_t)$ follows. Combined with $\triangle \oby^{k+1}\to0$, this yields $f_i(\by_i^{k+1})-f_i(\bx^{\tk})\to0$.
	From this result, we obtain
	\begin{eqnarray*}
		\arraycolsep=0.0pt\def\arraystretch{1.5}
		\begin{array}{lcl}
			F(Y^{k+1})\!-\!f(\bx^{\tk})\!=\!\sum_i\left( p_if_i(\by_i^{k+1})\!-\!p_i f_i(\bx^{\tk})\right)\!\to\!0.
		\end{array}
	\end{eqnarray*}
    
	c) Let $\ell:=(\tk-1)k_0 \in\mathcal{K}$. Then for any $k$,
	\begin{eqnarray}\label{tau-tau-s}
		\arraycolsep=1.0pt\def\arraystretch{1.5}
		\begin{array}{rcll}
			\ell+1&=& (\tk-1)k_0+1 \leq k+1 \leq \tk k_0,\\
			\tau_{\ell+1}&=&\lceil (\ell+1)/k_0 \rceil = \lceil  \tk-1+1/k_0 \rceil= \tk.
		\end{array}
	\end{eqnarray}
	Since $\ell\in\mathcal{K}$, it follows that
	\begin{eqnarray*}
		\arraycolsep=1.0pt\def\arraystretch{1.5}
		\begin{array}{rcll}
			\sum_i  \blam_i^{\ell+1}
			&\overset{\eqref{limit-5-term-0-inexact}}{\to}& \sum_i   ( \beta_i(\triangle \oby_i^{\ell+1}-\triangle \by_i^{\ell+1}+\triangle\obx^{\tau_{\ell+1}})-\triangle \blam_i^{\ell+1}+ \blam_i^{\ell+1} )
			=\sum_i   ( \beta_i(\by_i^{\ell}-\hat{\bx}^{\tau_{\ell+1}})+\blam_i^{\ell})\overset{\eqref{opt-con-xk1-1}}{=} 0.
		\end{array}
	\end{eqnarray*}
	Note that $\{\varepsilon_i^{k+1}\}$ is non-increasing. From \eqref{tau-tau-s}, we thus have $\varepsilon_i^{\ell+1} \geq \varepsilon_i^{k+1}$, which implies
	\begin{eqnarray*}
		\arraycolsep=1.0pt\def\arraystretch{1.5}
		\begin{array}{lcl}
			\|\blam_i^{k+1}-\blam_i^{\ell+1}\|^2
			&\overset{\eqref{opt-con-xk1-2}}{=}&   \| \Phi_i^{k+1}-\Phi_i^{\ell+1}-\bg_i^{k+1} +\bg_i^{\ell+1} \|^2  \\
			&\overset{\eqref{opt-con-xk1-2},\eqref{Lip-r}}{\leq}&     3{\varepsilon_i^{k+1}}+ 3{\varepsilon_i^{\ell+1}} + 3p_i^2L_i^2 \| \by_i^{k+1}-\by_i^{\ell+1}\|^2\\
			&\overset{\eqref{tau-tau-s}}{\leq}&   6{\varepsilon_i^{\ell+1}}  + 3p_i^2L_i^2 \| \by_i^{k+1}-\bx^{\tk} + \bx^{\tau_{\ell+1}}- \by_i^{\ell+1}\|^2\\
			&\overset{\eqref{two-vecs-3}}{\leq}&   6{\varepsilon_i^{\ell+1}}  + 6p_i^2L_i^2 (\| \triangle \oby^{k+1}_i\|^2 + \|\triangle\oby^{{\ell+1}}_i\|^2) \overset{\eqref{limit-5-term-0-inexact}}{\to}0.
		\end{array}
	\end{eqnarray*}
	Combining the above two results yields
	\begin{eqnarray}\label{limit-K-grad-pi-exact}
		\arraycolsep=1.0pt\def\arraystretch{1.5}
		\begin{array}{lcl}
			\sum_i\blam_i^{k+1}\to0.
		\end{array}
	\end{eqnarray}
	Taking the limit in \eqref{opt-con-xk1-2} yields
	\begin{eqnarray*}%\label{limit-K-grad-i-exact}
		\arraycolsep=1.0pt\def\arraystretch{1.5}
		\begin{array}{lcl}
			\nabla F(Y^{k+1}) &=& \sum_i  \bg_i^{k+1}
			\overset{\eqref{limit-K-grad-pi-exact}}\to\sum_i  (\bg_i^{k+1} +\blam_i^{k+1})
			\overset{\eqref{opt-con-xk1-2}}=\sum_i\Phi_i^{k+1} \overset{\eqref{opt-con-xk1-4}}\to0,
		\end{array}
	\end{eqnarray*}
	which, together with $\triangle\oby_i^{k+1}\to0$ and Assumption \ref{ass-fi}, implies 
	\begin{eqnarray*}
		\arraycolsep=1.0pt\def\arraystretch{1.5}
		\begin{array}{lcl}
	\nabla f (\bx^{\tk})=\sum_i  p_i \nabla f_i(\bx^{\tk})\to 0,
		\end{array}
\end{eqnarray*}
	thus completing the proof.
\end{proof}

\subsection{Proof of  Theorem \ref{global-convergence-inexact}}\label{theorem2}

\begin{proof}
	a)  Let $(\bx^{\infty},Y^{\infty},\Lambda^{\infty})$ be an accumulating point of the sequence. From \eqref{opt-con-xk1-2} and \eqref{opt-con-xk1-4}, it follows that
	$$0= p_i \nabla f_i(\by^{\infty}_i) + \blam_i^{\infty}.$$
	From $\triangle\oby_i^{k+1}\to0$ and \eqref{limit-K-grad-pi-exact}, we obtain
	\begin{eqnarray*}
		\arraycolsep=1.0pt\def\arraystretch{1.5}
		\begin{array}{lcl}
			0=\by^{\infty}_i- \bx^{\infty},~~ 0=\sum_i \blam^{\infty}_i.
		\end{array}
	\end{eqnarray*}
	Therefore, by \eqref{opt-con-FL-opt-ver1},  $(\bx^{\infty},Y^{\infty},\Lambda^{\infty})$  is a stationary point of   \eqref{FL-opt-ver1}, and $\bx^{\infty}$ is a stationary point of \eqref{FL-opt}.
	
	b) By \cite[Lemma 4.10]{more1983computing}, $\triangle \bx^{\tk} \to0$ and the isolation of $\bx^{\infty}$ imply $\bx^{\tk}\to\bx^{\infty}$. Combined with $\triangle\oby_i^{k+1}\to0$, this yields convergence of $\{Y^{k+1}\}$ to $Y^{\infty}$, i.e., $\by_i^{k+1}\to\bx^{\infty}$.
	Finally, since $\triangle \blam_i^{k+1}\to0$ and each limit point satisfies $\blam_i^\infty =-p_i\nabla f_i(\by_i^{\infty})$ with isolated structure, the sequence $\{\Lambda^{k+1}\}$ converges analogously to $\{\bx^{\tk}\}$.

\end{proof}

\subsection{Proof of  Corollary \ref{L-global-convergence}}\label{corollary1}
\begin{proof} a) The convexity of $f$ and optimality of $\bx^*$ imply
	\begin{align}\label{convexity-optimality}
		f(\bx^{\tau_k}) \geq  f(\bx^{*}) \geq  f(\bx^{\tau_k}) + \langle \nabla f(\bx^{\tau_k}), \bx^{*}{-}\bx^{\tau_k} \rangle.~
	\end{align}
	By Theorem \ref{global-obj-convergence-inexact} c),
	$$\lim_{k\to\infty} \nabla f(\bx^{\tau_k})=0.$$
	Combining this with the boundedness of $\{\bx^{\tau_k}\}$ from Theorem \ref{global-obj-convergence-inexact} a) and taking limits in \eqref{convexity-optimality} yields  $f(\bx^{\tk})\to f(\bx^*)$. By Theorem \ref{global-obj-convergence-inexact} b),  this implies \eqref{L-global-convergence-limit}.
	
	b) The result follows from  Theorem \ref{global-convergence-inexact} a) and the fact that, under convexity of $f$,  stationary points coincide with global solutions.
	
	c) The strong convexity of $f$ implies the existence of a constant $\theta$ such that
	\begin{eqnarray*}
		\arraycolsep=1.0pt\def\arraystretch{1.5}
		\begin{array}{lcl}
			f(\bx^{\tau_k}) -f( \bx^*)
			\geq  \langle \nabla f( \bx^*),\bx^{\tau_k}-\bx^*\rangle + \frac{\theta}{2}\|\bx^{\tau_k}-\bx^*\|^2
			=  \frac{\theta}{2}\|\bx^{\tau_k}-\bx^*\|^2,
		\end{array}
	\end{eqnarray*}
	where the equality follows from \eqref{grad-x-*=0}. Taking the limit in the above inequality and using $ f(\bx^{\tau_k}) \to f( \bx^*)$ implies $\bx^{\tau_k}\to\bx^*$. Combined with $\triangle\oby_i^{k+1}\to0$ in \eqref{limit-5-term-0-inexact}, this gives $ \by_i^k\to\bx^*$. Finally, $ \blam_i^k\to\blam_i^*$ follows from
	\begin{eqnarray*}
		\arraycolsep=1.0pt\def\arraystretch{1.5}
		\begin{array}{lcl}\|\blam_i^k-\blam_i^{*}\|^2&\overset{\eqref{opt-con-FL-opt-ver1},\eqref{opt-con-xk1-2}}{=}\|\Phi_i^k -( \bg_i^k - p_i \nabla f_i(\bx^*))\|^2
			 \overset{\eqref{Lip-r}}{\leq}&2{\varepsilon_i^k}+  2p_i^2L_i^2\| \by_i^k-\bx^*\|^2\to0,
		\end{array}
	\end{eqnarray*}
	displaying the desired result.
\end{proof}

\subsection{Proof of  Theorem \ref{complexity-thorem-gradient-inexact}}\label{theorem3}
\begin{proof}
	The proof considers $k\in\mathcal{K}$ and estimates $\|\nabla f(\bx^{\tau_{k +1}})\|^2 $ via the decomposition
	\begin{eqnarray}\label{decom-grad-f-tau}
		\arraycolsep=1pt\def\arraystretch{1.5}
		\begin{array}{rcll}
			\|\nabla f(\bx^{\tau_{k +1}})\|^2
			&\overset{\eqref{two-vecs-4}}\leq& 3\|\nabla f(\bx^{\tau_{k +1}})-\sum_i \bg_i^{k +1}\|^2+ 3\|\sum_i \bg_i^{k +1}-\sum_i \bg_i^{k }\|^2+3\|\sum_i \bg_i^{k }\|^2\\
			& \overset{\eqref{Lip-r}}{\leq}& N\sum_i 3p_i^2L_i^2(\| \triangle \oby_i^{{k +1}}\|^2+\| \triangle \by_i^{{k +1}}\|^2)+ 3\|\sum_i \bg_i^{k }\|^2\\
			& \overset{\eqref{sigma-a-r}}{\leq}& N\sum_i \frac{\beta_i^2}{3} (\| \triangle \oby_i^{{k +1}}\|^2+   \| \triangle \by_i^{{k +1}}\|^2)+3\|\sum_i \bg_i^{k }\|^2,
		\end{array}
	\end{eqnarray}
	We now bound each term on the right-hand side, as detailed below.
	
	\underline{Bound $\sum_i \|\beta _i\triangle \by_i^{k+1} \|^2$.}  From Lemma \ref{lemma-decreasing-1},
	\begin{eqnarray*}
		\arraycolsep=1.0pt\def\arraystretch{1.5}
		\begin{array}{lll}
			\sum_i\frac{\beta _i}{10}\|\triangle\by_i^{k+1}\|^2+\frac{\beta}{2}(1-a_{\tk}-a_{\tau_k})\|\triangle\hat{\bx}^{\tk}\|^2
			\leq \widetilde\L^k- \widetilde\L^{k+1},
		\end{array}
	\end{eqnarray*}
	with $\overline\beta:=\max_{i\in[N]}\beta_i$, this becomes
	\begin{eqnarray}\label{gap-w-w-L-L}
		\arraycolsep=1.0pt\def\arraystretch{1.5}
		\begin{array}{lll}
			\sum_i\|\beta_i\triangle \by_i^{k+1} \|^2\leq 10 \overline\beta  \sum_i \frac{\beta _i}{10}\|\triangle \by_i^{k+1} \|^2
			\leq10 \overline\beta( \widetilde\L^k- \widetilde\L^{k+1}).
		\end{array}
	\end{eqnarray}
    similarly, we get
	\begin{eqnarray}\label{gap-w-w-L-y}
		\arraycolsep=1.0pt\def\arraystretch{1.5}
		\begin{array}{lll}			\sum_i\|\beta_i\triangle\hat{\bx}^{\tk}\|^2\leq2\bar{\beta}\sum_i\frac{\beta_i}{2}\|\triangle\hat{\bx}^{\tk}\|^2
			\leq\frac{2\bar{\beta}}{1-a_{\tk}-a_{\tau_k}}(\widetilde\L^k- \widetilde\L^{k+1})
		\end{array}
	\end{eqnarray}
	Substituting \eqref{gap-w-w-L-y} into the following equation yields
	\begin{eqnarray*}
		\arraycolsep=1.0pt\def\arraystretch{1.5}
		\begin{array}{lll}
			\sum_i\| \beta_i\triangle\bx^{\tk}\|^2
			&\leq&\sum_i(a_{\tk}\|\beta_i(\hat{\bx}^{\tk}-\hat{\bx}^{\tk-1})\|+\|\beta_i\triangle \hat{\bx}^{\tk}\|+a_{\tau_k}\|\beta_i(\hat{\bx}^{\tau_k}-\hat{\bx}^{\tau_k-1})\|)^2\\
			&\leq&3a_{\tk}^2\sum_i\|\beta_i(\hat{\bx}^{\tk}-\hat{\bx}^{\tk-1})\|^2+3\sum_i\|\beta_i\triangle \hat{\bx}^{\tk}\|^2+3a_{\tau_k}^2\sum_i\|\beta_i(\hat{\bx}^{\tau_k}-\hat{\bx}^{\tau_k-1})\|^2\\
			&\leq&\frac{1+a_{\tk}^2}{1-a_{\tk}-a_{\tau_k}}6\bar{\beta}(\widetilde\L^k- \widetilde\L^{k+1})+\frac{a_{\tau_k}^2}{1-a_{\tau_k}-a_{\tau_{k-1}}}6\bar{\beta}(\widetilde\L^{k-1}- \widetilde\L^k)
		\end{array}
	\end{eqnarray*}
	If $a_{\tk}\in(0,\frac{1}{2})$ and $a_{\tk}\leq r$,
	$\frac{1+a_{\tk}^2}{1-a_{\tk}-a_{\tau_k}}\leq\frac{1+r^2}{1-2r}=R_1$, $\frac{a_{\tau_k}^2}{1-a_{\tau_k}-a_{\tau_{k-1}}}\leq\frac{r^2}{1-2r}=R_2$, and the preceding expression simplifies to
	\begin{eqnarray}\label{gap-hatx-hatx-L-L}
		\arraycolsep=1.0pt\def\arraystretch{1.5}
		\begin{array}{lll}
			\sum_i\| \beta_i\triangle\bx^{\tk}\|^2
			\leq6R_1\bar{\beta}(\widetilde\L^k- \widetilde\L^{k+1})+6R_2\bar{\beta}(\widetilde\L^{k-1}- \widetilde\L^k).
		\end{array}
	\end{eqnarray}
	
	\underline{Bound $\sum_i\|\beta_i \triangle\oby_i^{k +1}  \|^2$.}  For $i\in\tauk$, the third inequality in \eqref{gap-pi-k-eps} gives
	\begin{eqnarray} \label{opt-con-xk1-3-1}
		\arraycolsep=0.5pt\def\arraystretch{1.5}
		\begin{array}{lcl}
			\|\triangle   \blam_i^{k+1}\|^2
			&\overset{\eqref{gap-pi-k-eps}}{\leq}&   \frac{6p_i^2L_i^2}{5}  \|\triangle\by_i^{k+1} \|^2+12(\varepsilon_i^{k+1}+\varepsilon_i^k)\\
			&\overset{(\ref{sigma-a-r})}{\leq}&\frac{2\beta_i^2}{15} \|\triangle\by_i^{k+1} \|^2+12(\varepsilon_i^{k+1}+\varepsilon_i^k)\\
			&\overset{(\ref{iceadmm-sub20})}{\leq}&\frac{2\beta_i^2}{15} \|\triangle\by_i^{k+1} \|^2+ 24\varepsilon_i^k.
		\end{array}
	\end{eqnarray}
	Thus,
	\begin{eqnarray} \label{gap-xi-xtau-109}
		\arraycolsep=0.5pt\def\arraystretch{1.5}
		\begin{array}{lcl}
			\|\beta _i \triangle\oby_i^{k +1}\|^2 \overset{ \eqref{iceadmm-sub3}}{=} \|  \triangle \blam_i^{k +1} \|^2
			\overset{\eqref{opt-con-xk1-3-1}}{\leq}\frac{2\beta_i^2}{15} \|\triangle\by_i^{k +1} \|^2 +{24}\varepsilon_i^{k }.
		\end{array}
	\end{eqnarray}
	For  $i\notin\tauk$, let $k_i$ be defined as in the proof of Lemma \ref{basic-observations} a) at iteration $k$. Then $i\in\mathcal{N}^{\tau_{k_i+1}}$, and \eqref{opt-con-xk1-3-1} holds at $k_i$, from which we obtain
	\begin{eqnarray}\label{gap-xi-xtau-11}
		\arraycolsep=1.0pt\def\arraystretch{1.5}
		\begin{array}{lcl}
			\|\beta_i \triangle\oby_i^{k +1}  \|^2
			&\overset{\eqref{gap-xi-xtau}}\leq& 2\| \triangle \blam_i^{k_i+1} \|^2+2s_0\beta_i^2\sum\limits_{t=\tau_{k_i+1}}^{\tk-1 }\| \bx^{ {t+1}}-\bx^{ {t}}\|^2\\
			&\overset{\eqref{opt-con-xk1-3-1}}\leq& \frac{4\beta_i^2}{15} \|\triangle\by_i^{k_i+1} \|^2 + 48\varepsilon_i^{k_i}+2s_0\beta_i^2\sum\limits_{t=\tau_{k_i+1}}^{\tk -1}\| \bx^{ {t+1}}-\bx^{ {t}}\|^2.
		\end{array}
	\end{eqnarray}
	By Scheme \ref{ass-omega}, $\tau_{k +1} -\tau_{k_i+1}\leq s_0$, so $k\geq k_i \geq k -(s_0-1)k_0-1\geq k -s_0k_0$. The non-increasing properties of $\{\widetilde\L^k\}$ and $\{\varepsilon_i^k\}$ then yield
	\begin{align} \label{Lki-Lks0}
		\widetilde\L^k \leq \widetilde\L^{k_i}\leq\widetilde\L^{k -s_0k_0}, \varepsilon_i^k\overset{\eqref{iceadmm-sub5}}=\varepsilon_i^{k_i}, i\notin\tauk.
	\end{align}
	Combining \eqref{gap-xi-xtau-109} and \eqref{gap-xi-xtau-11} yields
	\begin{eqnarray} \label{gap-xi-xtau-112}
		\arraycolsep=0.5pt\def\arraystretch{1.75}
		\begin{array}{lcl}
			&&\sum_i\|\beta_i \triangle\oby_i^{k +1}  \|^2\\ 
			&=&\sum\limits_{i\in\tauk}\|\beta_i \triangle\oby_i^{k +1}  \|^2 + \sum\limits_{i\notin\tauk}\|\beta_i \triangle\oby_i^{k +1}  \|^2\\
			&\leq&\sum\limits_{i\in\tauk}(\frac{2\beta_i^2}{15} \|\triangle\by_i^{k +1} \|^2 +{24}\varepsilon_i^{k })+ \sum\limits_{i\notin\tauk} (\frac{4\beta_i^2}{15} \|\triangle\by_i^{k_i+1} \|^2 + 48\varepsilon_i^{k_i}+ 2s_0\beta_i^2\sum\limits_{t=\tau_{k_i+1}}^{\tk -1}\| \bx^{ {t+1}}-\bx^{ {t}}\|^2 )\\
			&\leq&\sum_i(\frac{2\beta_i^2}{15} \|\triangle\by_i^{k +1} \|^2 + \frac{4\beta_i^2}{15} \|\triangle\by_i^{k_i+1} \|^2 )+ \sum_i  (2s_0\beta_i^2\sum\limits_{t=\tau_{k_i+1}}^{\tk -1}\| \bx^{ {t+1}}-\bx^{ {t}}\|^2+ {48} \varepsilon_i^k).
		\end{array}
	\end{eqnarray}
    
	\underline{Bound $\|\sum_i \bg_i^{k }\|^2 $.} For $k \in\mathcal{K}$,
	\begin{eqnarray*}
		\arraycolsep=1.0pt\def\arraystretch{1.5}
		\begin{array}{rcll}
			\sum_i \bg_i^{k } \overset{\eqref{opt-con-xk1-2}}= \sum_i (\Phi_i^{k }-  \blam_i^{k })\overset{\eqref{opt-con-xk1-1}}= \sum_i (\Phi_i^{k } + \beta_i(\by_i^{k }-\bx^{\tau_{k +1}}))
			= \sum_i (\Phi_i^{k } + \beta_i( \triangle \oby_i^{k +1}- \triangle \by_i^{k +1})),
		\end{array}
	\end{eqnarray*}
	and since $\|\Phi_i^k\|^2\leq \varepsilon_i^k $ by \eqref{opt-con-xk1-2}, it follows that
	\begin{eqnarray} \label{sum-g-i}
		\arraycolsep=0.5pt\def\arraystretch{1.75}
		\begin{array}{lcl}
			\|\sum_i \bg_i^{k }\|^2
			\overset{\eqref{two-vecs-4}}{\leq}   3 N \sum_i \varepsilon_i^{k }+ 3N\sum_i \beta_i^2\|  \triangle \oby_i^{k +1}\|^2+3N\sum_i \beta_i^2\| \triangle \by_i^{k +1}\|^2.
		\end{array}
	\end{eqnarray}
	Now, \eqref{decom-grad-f-tau} and  \eqref{sum-g-i}  together imply
	\begin{eqnarray}\label{error-grad-f-w}
		\arraycolsep=0pt\def\arraystretch{1.5}
		\begin{array}{rcll}
			&&\|\nabla f(\bx^{\tau_{k +1}})\|^2\\
			&\leq&\sum_i \frac{28N\beta_i^2}{3} (\| \triangle \oby_i^{{k +1}}\|^2{+} \| \triangle \by_i^{{k +1}}\|^2) + 9 N \sum_i \varepsilon_i^{k }\\
			&\overset{\eqref{gap-xi-xtau-112}}\leq&
			\frac{476N}{45}\sum_i  \|\beta_i\triangle\by_i^{k +1} \|^2
			+ \frac{112N}{45}\sum_i \|\beta_i\triangle\by_i^{k_i+1} \|^2+ {457N}\sum_i   \varepsilon_i^{ k}
			+  \frac{56Ns_0}{3}\sum\limits_{t=\tau_{k_i+1}}^{\tk -1} \sum_i \|\beta_i( \bx^{ {t+1}}-\bx^{ {t}})\|^2\\
			&\overset{\eqref{gap-w-w-L-L}}\leq& \frac{952N\overline{\beta}}{9}(\widetilde\L^k- \widetilde\L^{k+1})+\frac{224N\overline{\beta}}{9}(\widetilde\L^{k_i}- \widetilde\L^{k_i+1})+  {457N}\sum_i   \varepsilon_i^{ k}
			+  \frac{56Ns_0}{3}\sum\limits_{t=\tau_{k_i+1}}^{\tk -1} \sum_i \|\beta_i( \bx^{ {t+1}}-\bx^{ {t}})\|^2\\
			&\overset{\eqref{gap-hatx-hatx-L-L}}\leq& \frac{952N\overline{\beta}}{9}(\widetilde\L^k- \widetilde\L^{k+1})+\frac{224N\overline{\beta}}{9}(\widetilde\L^{k_i}- \widetilde\L^{k_i+1})+  {457N}\sum_i   \varepsilon_i^{ k}\\
			&&+  112Ns_0\overline{\beta}\sum\limits_{t=k_i+1}^k[R_1(\widetilde\L^{t}- \widetilde\L^{t+1})+R_2(\widetilde\L^{t-1}- \widetilde\L^{t})]\\
			&\leq& \frac{952N\overline{\beta}}{9}(\widetilde\L^k- \widetilde\L^{k+1})+{457N}\sum_i   \varepsilon_i^{ k}+112Ns_0R_1\overline{\beta}(\widetilde\L^{k_i}- \widetilde\L^{k+1})+112Ns_0R_2\overline{\beta}(\widetilde\L^{k_i}- \widetilde\L^k)\\
			&\leq& \frac{952N\overline{\beta}}{9}(\widetilde\L^k- \widetilde\L^{k+1})+{457N}\sum_i   \varepsilon_i^k+112Ns_0R\overline{\beta}(\widetilde\L^{k_i}- \widetilde\L^{k+1})\\
			&\overset{\eqref{Lki-Lks0} }{\leq}&  \frac{1960N s_0R\overline{\beta}}{9}(\widetilde\L^{ k -s_0k_0}- \widetilde\L^{k+1})+    {457N}\sum_i   \varepsilon_i^{ k}.
		\end{array}
	\end{eqnarray}
	where $R=R_1+R_2$. The non-increasing property of $\{\widetilde\L^k\}$  implies
	\begin{align}\label{L-2s-L-2s}
		\widetilde\L^{(p+1)k_0}-\widetilde\L^{pk_0+1}  \leq0, ~~\forall p\geq 0.
	\end{align}
	Note that $\{\varepsilon_i^k\}$ is non-increasing. Moreover, Scheme \ref{ass-omega} ensures that within every  $s_0k_0$ steps, each client $i\in[N]$ is selected at least once to update its parameters via \eqref{iceadmm-sub20}-\eqref{iceadmm-sub4}, so $\varepsilon_i^{(a+1)s_0k_0}\leq \theta_i \varepsilon_i^{as_0k_0}$ for all $i\in[N]$ and $a=1,2,\cdots$, and thus
	\begin{eqnarray}\label{sum-eps-i}
		\arraycolsep=0pt\def\arraystretch{1.5}
		\begin{array}{rcll}
			\sum\limits_{p=s_0}^{t} \varepsilon_i^{pk_0}  \leq\sum\limits_{a=1}^{\lceil  {t}/{s_0}\rceil } s_0 \varepsilon_i^{as_0k_0}
			\leq   \frac{s_0\varepsilon_i^{s_0k_0} (1-\theta_i^{\lceil  {t}/{s_0}\rceil})}{1-\theta_i} \leq \frac{s_0\varepsilon_i^0}{1-\theta_i}.
		\end{array}
	\end{eqnarray}
	Finally, \eqref{L-2s-L-2s} and \eqref{sum-eps-i} together yield
	\begin{eqnarray*}
		\arraycolsep=0pt\def\arraystretch{1.5}
		\begin{array}{rcll}
			&&(t-s_0+1)\min\limits_{p=s_0,s_0+1,\cdots,t}\|\nabla f(\bx^{\tau_{pk_0 +1}})\|^2\\
			&\leq&  \sum\limits_{p=s_0}^{t} \|\nabla f(\bx^{\tau_{pk_0 +1}})\|^2\\
			&\overset{\eqref{error-grad-f-w}}{\leq}&  \sum\limits_{p=s_0}^{t}  \frac{1960N s_0R\overline{\beta}}{9}(\widetilde\L^{(p-s_0)k_0}- \widetilde\L^{pk_0+1})+ \sum\limits_{p=s_0}^{t} {457N}\sum_i   \varepsilon_i^{pk_0} \\
			&\overset{\eqref{L-2s-L-2s},\eqref{sum-eps-i}}\leq& \frac{1960N s_0R\overline{\beta}}{9}[(\widetilde\L^{0}+\widetilde\L^{k_0}+\cdots + \widetilde\L^{s_0k_0})-(\widetilde\L^{(t-s_0)k_0+1}+\widetilde\L^{(t-s_0+1)k_0+1}+\cdots+\widetilde\L^{tk_0+1})]+\sum_i \frac{457N s_0\varepsilon_i^0}{ 1-\theta_i }\\
			&\leq&    \frac{1960N s_0(s_0+1)R\overline{\beta}}{9}   (\widetilde\L^{0} - f(\bx^*))+\sum_i \frac{457Ns_0\varepsilon_i^0}{1-\theta_i}= c,
		\end{array}
	\end{eqnarray*}
	where the last inequality follows from $ \widetilde\L^0\geq\widetilde\L^k \geq f(\bx^*)$ for all $ k\geq1$ by Lemma \ref{L-bounded-decreasing-1} a)b). Setting $k = t k_0$ with $k > 2 s_0 k_0$, we obtain
	\begin{align*}
		\min_{s=1,2,\cdots,k}\|\nabla f(\bx^{\tau_{s +1}})\|^2
		\leq\min_{p=s_0,s_0+1,\cdots,t}\|\nabla f(\bx^{\tau_{pk_0 +1}})\|^2
		\leq\frac{c}{t-s_0+1} \leq \frac{c k_0}{k -s_0k_0}\leq \frac{2c k_0}{k},
	\end{align*}
	establishing the desired result.
\end{proof}
\newpage
\bibliographystyle{abbrv}	
\bibliography{ref}

\end{document}